\documentclass[10pt,journal]{IEEEtran}
\ifCLASSINFOpdf
  \usepackage[pdftex]{graphicx}
\else
\fi
\usepackage{amsmath}
\usepackage{amssymb}
\usepackage{amsthm}%Cui Zhen
\usepackage{cases}%Cui Zhen
\usepackage{bm}%Cui Zhen
\usepackage{algorithm}%Cui Zhen
\usepackage{algorithmic}%Cui Zhen
\usepackage{cases}%Cui Zhen
\usepackage{subfigure}%%cuizhen
\usepackage{multirow}%%cuizhen
\usepackage{xcolor} %% cuizhen
\usepackage{enumerate}%% cuizhen
\usepackage{etoolbox}%% cuizhen
\usepackage{boxedminipage} %% cuizhen
\usepackage{booktabs} %% cuizhen
\newcommand{\MyMapTemplatePrefix}[4]{\expandafter#1\csname#3#4\endcsname{#2{#4}}}
\newcommand{\MyMapTemplatePrefixNew}[5]{\expandafter#1\csname#4#5\endcsname{#2{#3{#5}}}}
\forcsvlist{\MyMapTemplatePrefix {\def} {\mathbf} {}} {A,B,C,D,E,F,G,H,I,J,K,L,M,N,O,P,Q,R,S,T,U,V,W,X,Y,Z}
\forcsvlist{\MyMapTemplatePrefix {\def} {\mathbf} {}} {a,b,c,d,e,f,g,h,i,j,k,l,m,n,o,p,q,r,s,t,u,v,w,x,y,z,1,0}
\forcsvlist{\MyMapTemplatePrefix {\def} {\widetilde} {wt}} {A,B,C,D,E,F,G,H,I,J,K,L,M,N,O,P,Q,R,S,T,U,V,W,X,Y,Z}
\forcsvlist{\MyMapTemplatePrefix {\def} {\widetilde} {wt}} {a,b,c,d,e,f,g,h,i,j,k,l,m,n,o,p,q,r,s,t,u,v,w,x,y,z} % r
\forcsvlist{\MyMapTemplatePrefixNew {\def} {\widetilde}{\mathbf} {tb}} {A,B,C,D,E,F,G,H,I,J,K,L,M,N,O,P,Q,R,S,T,U,V,W,X,Y,Z}
\forcsvlist{\MyMapTemplatePrefixNew {\def} {\widetilde}{\mathbf} {tb}} {a,b,c,d,e,f,g,h,i,j,k,l,m,n,o,p,q,r,s,t,u,v,w,x,y,z}
\forcsvlist{\MyMapTemplatePrefix {\def} {\widehat} {wh}} {A,B,C,D,E,F,G,H,I,J,K,L,M,N,O,P,Q,R,S,T,U,V,W,X,Y,Z}
\forcsvlist{\MyMapTemplatePrefix {\def} {\widehat} {wh}} {a,b,c,d,e,f,g,h,i,j,k,l,m,n,o,p,q,r,s,t,u,v,w,x,y,z}
\forcsvlist{\MyMapTemplatePrefixNew {\def} {\widehat}{\mathbf} {hb}} {A,B,C,D,E,F,G,H,I,J,K,L,M,N,O,P,Q,R,S,T,U,V,W,X,Y,Z}
\forcsvlist{\MyMapTemplatePrefixNew {\def} {\widehat}{\mathbf} {hb}} {a,b,c,d,e,f,g,h,i,j,k,l,m,n,o,p,q,r,s,t,u,v,w,x,y,z}
\forcsvlist{\MyMapTemplatePrefix {\def} {\mathcal}{mc}} {A,B,C,D,E,F,G,H,I,J,K,L,M,N,O,P,Q,R,S,T,U,V,W,X,Y,Z}
\forcsvlist{\MyMapTemplatePrefix {\def} {\mathbb} {mb}} {A,B,C,D,E,F,G,H,I,J,K,L,M,N,O,P,Q,R,S,T,U,V,W,X,Y,Z}
\forcsvlist{\MyMapTemplatePrefix {\DeclareMathOperator} {} {} } {tr,diag,sgn}
\def\tp{^\top}
\def\st{\text{s.t.~}}
\def\ie{\emph{i.e.}}
\def\etal{\emph{et al.}}

\def\eg{\emph{e.g.}}

\newtheorem{thm}{Theorem}%[section]
\newcommand{\tabincell}[2]{\begin{tabular}{@{}#1@{}}#2\end{tabular}}%%

\begin{document}
%
% paper title
% Titles are generally capitalized except for words such as a, an, and, as,
% at, but, by, for, in, nor, of, on, or, the, to and up, which are usually
% not capitalized unless they are the first or last word of the title.
% Linebreaks \\ can be used within to get better formatting as desired.
% Do not put math or special symbols in the title.
\title{Action-Attending Graphic Neural Network}
%
%
% author names and IEEE memberships
% note positions of commas and nonbreaking spaces ( ~ ) LaTeX will not break
% a structure at a ~ so this keeps an author's name from being broken across
% two lines.
% use \thanks{} to gain access to the first footnote area
% a separate \thanks must be used for each paragraph as LaTeX2e's \thanks
% was not built to handle multiple paragraphs
%

\author{Chaolong~Li*,
        Zhen~Cui*,~\IEEEmembership{Member,~IEEE,}
        Wenming~Zheng,~\IEEEmembership{Member,~IEEE,}
        Chunyan~Xu,~\IEEEmembership{Member,~IEEE,}
        Rongrong~Ji,~\IEEEmembership{Senior Member,~IEEE,}
        and~Jian Yang,~\IEEEmembership{Member,~IEEE}% <-this % stops a space
\thanks{* C. Li and Z. Cui have equal contributions.}
\thanks{C. Li and W. Zheng are with the Key Laboratory of Child Development and Learning Science (Southeast University), Ministry of Education, School of Biological Science \& Medical Engineering, Southeast University, Nanjing 210096, China (e-mail: \{lichaolong, wenming\_zheng\}@seu.edu.cn).}% <-this % stops a space
\thanks{Z. Cui, C. Xu and J. Yang are with the School of Computer Science and Engineering, Nanjing University of Science and Technology, Nanjing 210094, China (e-mail: \{zhen.cui, cyx, csjyang\}@njust.edu.cn).}
\thanks{R. Ji is with the School of Information Science and Engineering, Xiamen University, Xiamen 361005, China (e-mail: rrji@xmu.edu.cn).}% <-this % stops a space
}
\maketitle

% As a general rule, do not put math, special symbols or citations
% in the abstract or keywords.
\begin{abstract}
  The motion analysis of human skeletons is crucial for human action recognition, which is one of the most active topics in computer vision. In this paper, we propose a fully end-to-end action-attending graphic neural network (A$^2$GNN) for skeleton-based action recognition, in which each irregular skeleton is structured as an undirected attribute graph. To extract high-level semantic representation from skeletons, we perform the local spectral graph filtering on the constructed attribute graphs like the standard image convolution operation. Considering not all joints are informative for action analysis, we design an action-attending layer to detect those salient action units (AUs) by adaptively weighting skeletal joints. Herein the filtering responses are parameterized into a weighting function irrelevant to the order of input nodes. To further encode continuous motion variations, the deep features learnt from skeletal graphs are gathered along consecutive temporal slices and then fed into a recurrent gated network. Finally, the spectral graph filtering, action-attending and recurrent temporal encoding are integrated together to jointly train for the sake of robust action recognition as well as the intelligibility of human actions. To evaluate our A$^2$GNN, we conduct extensive experiments on four benchmark skeleton-based action datasets, including the large-scale challenging NTU RGB+D dataset. The experimental results demonstrate that our network achieves the state-of-the-art performances.
\end{abstract}

% Note that keywords are not normally used for peerreview papers.

% For peer review papers, you can put extra information on the cover
% page as needed:
% \ifCLASSOPTIONpeerreview
% \begin{center} \bfseries EDICS Category: 3-BBND \end{center}
% \fi
%
% For peerreview papers, this IEEEtran command inserts a page break and
% creates the second title. It will be ignored for other modes.
\IEEEpeerreviewmaketitle

\section{Introduction}

\IEEEPARstart{H}{uman} action recognition is an active area of research with wide applications such as video surveillance, games console, robot vision, etc. Over the past few decades, human action recognition from 2D RGB video sequences has been extensively studied~\cite{aggarwal2011human}. However, 2D cameras cannot fully capture human motions, which are actually located in 3D space. With the advent of depth sensors such as Microsoft Kinect and Asus Xtion PRO LIVE, 3D action recognition arises more attention of researchers due to its several advantages in segmenting foreground/background, resisting illumination variations, etc.

Recently several approaches have been developed to deal with 3D human action recognition. Generally they fall into two categories: depth map based or skeleton based. The depth map based methods directly extract volumetric and temporal features of overall point set from depth map sequences. The skeleton based methods utilize 3D coordinates of skeletal joints estimated from depth maps to model actions of human body. As human body can be viewed as an articulated system of rigid bones connected by hinged joints, the actions of human body principally reflect in human skeletal motions in the 3D space~\cite{ye2013survey}. Thereby, to model human skeletal joints should be more effective for action recognition, as suggested in the early work~\cite{johansson1975visual}.

Skeleton based methods~\cite{amor2016action,cai2016effective,devanne20153,ellis2013exploring,slama2015accurate,tao2015moving,vemulapalli2014human,wang2014learning,xia2012view,wang2016graph} have become prevalent for human action recognition since Shotton \etal~\cite{shotton2011real} successively estimated 3D coordinates of skeletal joints from depth maps. To represent trajectories of human body motions, the position, speed or even acceleration of skeletal joints are often gathered from several consecutive frames, and then modeled with some statistic methods (\eg, histogram). Further, the skeletal joints may be partitioned into several parts according to the concurrent function of adjacent joints, so human actions are represented with the motion parameters of body parts. To encode temporal motions of skeletal joints, linear/non-linear dynamical systems are often used to model human action, \eg, Hidden Markov Models (HMMs)~\cite{xia2012view} and Long Short Term Memory (LSTM)~\cite{du2016representation}. However, there still exist some key issues need to be studied deeply. First, how to extract robust high-level semantic features from irregular skeletons? The current skeleton-based methods usually employ some simple statistical features on skeletal joints or body parts. It is insufficient to describe and abstract spatial structure of skeleton at each time slice, whereas non-linear dynamic systems are only used to abstract high-level motion information. Second, which/what action units (AUs) identify a special human action? Most motions are produced from only a few joints or body parts, \eg, waving with arm and hand, kicking ball with leg and foot. Detecting salient action units should be helpful to eliminate some useless motion noises as well as provide a cognitive explanation to understand human action.

To address the above problems, in this paper, we propose a fully end-to-end action-attending graphic neural network (A$^2$GNN) for skeleton-based action recognition. To extract high-level semantic information from spatial skeletons, we represent each human skeleton with an undirected attribute graph, and then perform the local spectral filtering on the structured graph to laywisely abstract spatial skeletal information like the classic convolutional neural network (CNN)~\cite{krizhevsky2012imagenet}. Hence, different from the recent graph-based method~\cite{wang2016graph}, which only took the traditional technique line of graph matching, A$^2$GNN should be a true deep network directly learning from irregular skeletal structure. To detect those salient action units, we design an action-attending layer to adaptively weight skeletal joints for different human motions. Specifically, we parameterize the weighting process as one dynamic function, which takes the responses of local spectral filtering on skeletal graphs as the input. Incidentally, we draw a conclusion that the weighting process is irrelevant to the input order of skeletal joints, which thus can be used for those general graph tasks. After extracting high-level discriminant features at each time slice, we finally stack a recurrent neural network on a consecutive sequence to model temporal variations of different human actions. All processes including spectral graph filtering, action unit detection, temporal motion modeling are integrated into one network framework to train jointly. To evaluate our proposed method, we conduct extensive experiments on four benchmark skeleton-based action datasets: the Motion Capture Database HDM05~\cite{muller2007documentation}, the Florence 3D Action dataset~\cite{seidenari2013recognizing}, the Large Scale Combined (LSC) dataset~\cite{zhang2016}, the NTU RGB+D dataset~\cite{shahroudy2016ntu}. The experimental results demonstrate our A$^2$GNN is competitive over the state-of-the-art methods.

In summary, our contributions are three folds:
\begin{enumerate}
  \item Propose a fully end-to-end graphical neural network framework to deal with skeleton-based action recognition, where we model human skeletons as attribute graphs and then introduce spectral graph filtering to extract high-level skeletal features, .
  \item Design a weighting way to adaptively detect salient action units for different human actions, which not only promotes human action recognition accuracy but also favors our cognitive understanding to human actions.
  \item Achieve the state-of-the-art performances on the four benchmark datasets including the two large-scale challenging datasets: LSC and NTU RGB+D.
\end{enumerate}

The reminder of this paper is organized as follows. In Section~\ref{sec:related-work}, we introduce related work on skeleton based action recognition. In Section~\ref{sec:proposed-model}, we first give an overview of our proposed network and introduce three main modules respectively. Implementation details of our proposed network are stated in Section~\ref{sec:implementation}. Experimental results and discussion are presented in Section~\ref{sec:experiments}. Finally, we conclude this work in Section~\ref{sec:conclusion}.

\section{Related Work}\label{sec:related-work}

The most related works to ours are those methods of skeleton-based human action recognition. Here we briefly review them from the view of skeletal representation, including low-level statistical features and high-level semantic features.

Various low-level statistical features have been used to describe skeleton data in the past years. Generally they contain three categories: joint based, body part based and pose based features. Hussein \etal~\cite{hussein2013human} used covariance matrix of skeletal joint coordinates over time to describe motion trajectories. In the literature~\cite{gowayyed2013histogram}, histogram of oriented displacement was used to represent the 3D trajectories of body joints. Ofli \etal~\cite{ofli2014sequence} chose a few most informative joints, and employed some physical interpretable measures, such as the mean or variance of joint angles, maximum angular velocity of joints and so on, to encode skeletal motions. Considering joint-associated movements, Chaudhry \etal~\cite{chaudhry2013bio} divided human skeleton into several smaller parts hierarchically and depicted each part with certain bio-inspired shapes. Vemulapalli \etal~\cite{vemulapalli2014human} formulated each body part over times into a curved manifold and performed action recognition in the Lie group. In view of skeletal postures, Zanfir \etal~\cite{zanfir2013moving} proposed the moving pose descriptor by considering position, speed and acceleration information of body joints. Xia \etal~\cite{xia2012view} represented postures as histograms of 3D skeletal joint locations by casting  selected joints into corresponding spatial histogram bins. These low-level statistical features are usually limited in representing those complex human actions.

High-level sematic features are popular for encoding human actions due to the robust representation ability. Especially, the temporal pyramid is used to hierarchically encode temporal dynamical variations, and all level features are gathered together to model actions~\cite{tao2015moving,vemulapalli2014human,gowayyed2013histogram,Luo2013Group,wang2012mining}. To model the temporal dynamics, Vemulapalli \etal~\cite{vemulapalli2014human} used dynamic time warping (DTW) and Fourier temporal pyramid (FTP); Wang \etal~\cite{wang2012mining} employed FTP to encoded local occupancy patterns over times; Chaudhry \etal~\cite{chaudhry2013bio} employed linear dynamical systems (LDSs) to learn the dynamical variation of features; Xia \etal~\cite{xia2012view} used HMMs to capture the temporal action dynamics. In addition, more recently, recurrent neural network has been employed to encode the temporal dynamic variations~\cite{du2016representation,baccouche2011sequential}. However, these methods mainly focus on the deep encoding of temporal motions rather than spatial layout of joints.

Until recently, graph was used to represent skeletal motion in the literature~\cite{wang2016graph}, although a few graph-based algorithms~\cite{gaur2011string,wang2013directed} had been proposed for action recognition on RGB videos. These graph-based methods usually took the traditional graph matching strategy after modeling skeletal joints or body parts into the graph structure. Therefore, two crucial issues, the construction of graph and the definition of graph kernel, need be conducted in these graph-based methods. Different from them, we perform spectral graph filtering on skeleton-induced graphs to extract high-level skeletal features from the spatial graphs. With a combination of recurrent motion encoding, the spatial-temporal features of skeletal motions are abstracted from the constructed end-to-end network. In contrast to the existing skeleton-based action recognition methods, another one important difference is that an adaptive action-attending mechanism is introduced to detect those salient action units w.r.t different human actions, which can benefit the final human action recognition.

\section{The Proposed A$^2$GNN}\label{sec:proposed-model}

In this section, we first give an overview of our proposed A$^2$GNN, then we respectively introduce three involved modules: the learning of deep graphical features, the detection of salient action units and the dynamic modeling of temporal motions.

\begin{figure*}[t]
  \centering
  \includegraphics[width=6.8in]{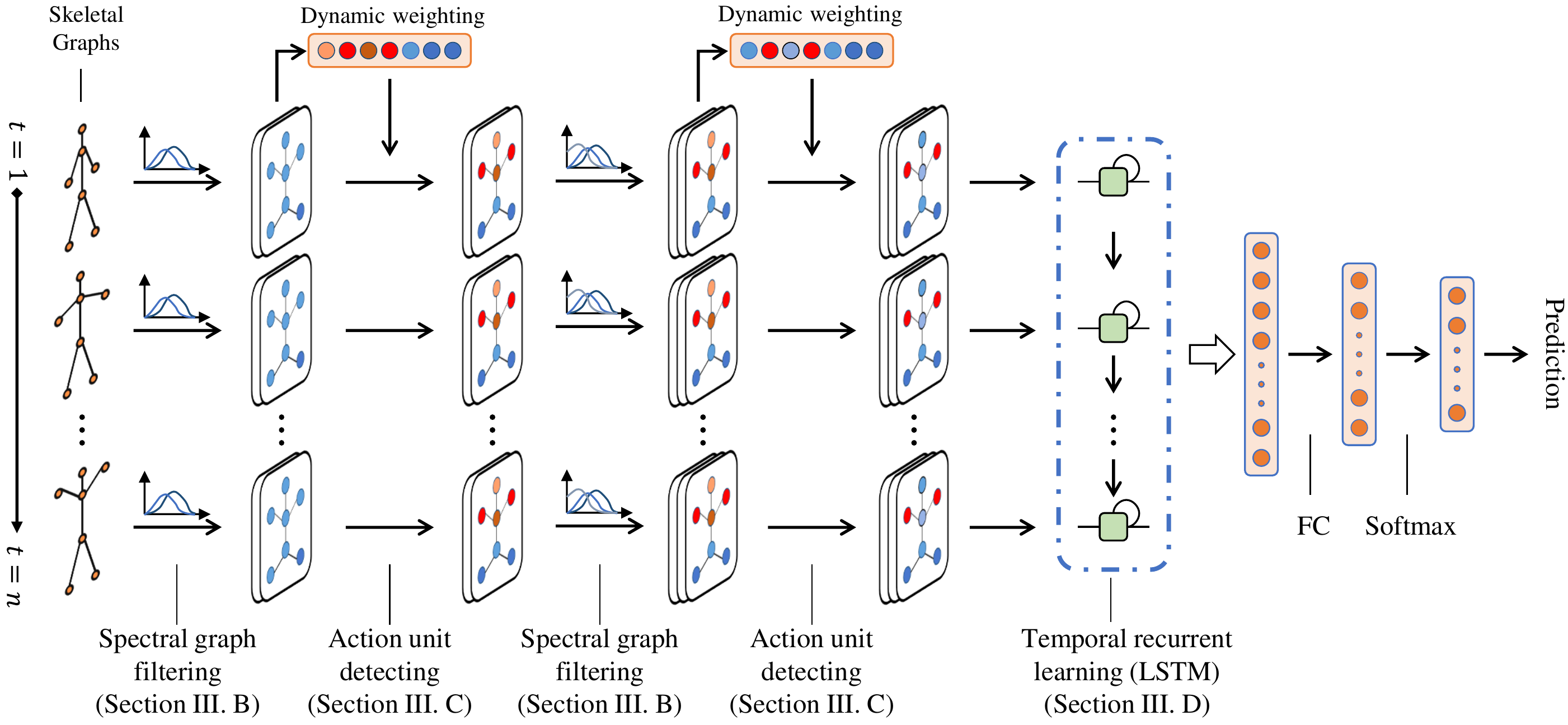}
  \caption{The illustration of our proposed A$^2$GNN architecture. An overall introduction can be found in Section~\ref{sec:overview}.}
  \label{fig:architecture}
\end{figure*}

\subsection{Overview}\label{sec:overview}

An overview of our proposed A$^2$GNN is illustrated in Fig.~\ref{fig:architecture}. The input is a motion sequence of skeletons, in which each skeletal joint is described as a 3D coordinate $(x,y,z)$. For each skeleton at one time slice, we model it into an undirected attribute graph, where each skeletal joint is one node and the bone between two joints is considered as one connected edge. For the weight between two nodes, we assign a connected edge to 1, otherwise 0. In addition, several alternative ways to weight edges are permissible, \eg, Gaussian kernel. Another important characteristic is that each node is associated with an observed signal vector (a.k.a. attribute), which is 3D spatial coordinates of the joint. To reduce individual differences of different samples, we normalize the input signals with coordinating, scaling and rotating transformations.

For the constructed spatial skeletal graphs, in order to extract high-level skeletal features, we expect to perform convolutional filtering on them like on regular grid-shape images. To the end, we introduce local spectral filtering on graphs, inspired by signal theory on graphs~\cite{shuman2013emerging} and the recent graph convolution~\cite{defferrard2016convolutional}. To avoid the eigenvalue decomposition of Laplacian matrix, the original solution is approximated by a polynomial of Laplacian matrix, in which each $k$-order term derives a $k$-hop neighborhood subgraph like a local receptive field. The details are introduced in Section~\ref{sec:gcnn}.

Specially, we design an action-attending layer to detect salient action units by adaptively weighting skeletal joints for different human actions. Usually, specific action units are activated for different human actions, \eg, hands and arms segments for clapping, hand, arm and head segments for drinking. Hence, weighting skeletal joints may reduce the disturbance of those useless joints, and benefit the final action recognition. The related details may be found in Section~\ref{sec:detect_au}.

By alternately stacking the spectral graph filtering layer and the action-attending layer, we may obtain the graph features of skeletal joints. After concatenating features of all joints at one time slice, we fed it into a recurrent network (LSTM) to encode feature variations at all temporal slices. Please refer to more details in Section~\ref{sec:lstm}. Finally, a fully connected layer is used to gather the outputs of recurrent network and learn the skeleton sequence representation followed by a softmax layer for classification. All processes including spectral graph filtering, action unit detection, temporal motion modeling are integrated into a network framework and jointly trained.

\subsection{Learning Deep Graphical Features}\label{sec:gcnn}

We model a body skeleton into an undirected attribute graph $\mcG=(\mcV, \A, \X)$ of $N$ nodes, where $\mcV=\{v_1,...,v_N\}$ is the set of skeletal joints, $\A$ is the (weighted) adjacency matrix, and $\X$ is the matrix of node signals/attributes. The adjacency matrix $\A\in\mbR^{N\times N}$ encodes the connection between two nodes (or joints), where if $v_i, v_j$ are connected, then $A_{ij}=1$, otherwise $A_{ij}=0$. If each joint is endowed with a vectorized signal of 3D coordinates, \ie, $\x:\mcV\rightarrow\mbR^3$, we may stack the signals of all nodes to form the signal matrix $\X\in\mbR^{N\times 3}$, where each row is associated with one node.

In order to extract skeletal graph features, we aim to perform convolutional filtering on these irregular attribute graphs like on regular grid-shape images. As studied in~\cite{bruna2013spectral,niepert2016learning}, it is difficult to express a meaningful translation operator in the vertex domain. According to the spectral graph theory~\cite{shuman2013emerging}, the convolutional filtering on graphs depends on
the graph Laplacian operator $\mcL = \D-\A$, where $\D\in\mbR^{N\times N}$ is the diagonal degree matrix with $D_{ii}=\sum_{j}A_{ij}$. For the graph Laplacian matrix, the normalized version is often used, \ie,
\begin{align}
\mcL^{norm} = \D^{-\frac{1}{2}}\mcL\D^{-\frac{1}{2}}= \I-\D^{\frac{1}{2}}\W\D^{\frac{1}{2}},
\end{align}
where $\I$ is the identity matrix. Unless otherwise specified, we use the normalized Laplacian matrix below.

As a real symmetric positive definite (SPD) matrix, the graph Laplacian matrix $\mcL$ may be decomposed into
\begin{align}
\mcL=\U\Lambda\U\tp,\label{eqn:L_SVD}
\end{align}
where $\Lambda=\diag([\lambda_1,\lambda_2,\cdots,\lambda_N])$ is a diagonal matrix of nonnegative real eigenvalues $\{\lambda_l\}$ (a.k.a spectrum), and the orthogonal matrix $\U=[\u_1,\cdots,\u_N]$ means the corresponding eigenvectors. Analogous to the classic Fourier transform, the graph Fourier transform of a signal $\x$ in spatial domain can be defined as $\hbx = \U\tp\x$, where $\hbx$ is the produced frequency signal. The corresponding inverse Fourier transform is $\x = \U\hbx$.

Give any one filtering function $g(\cdot)$ of the graph $\mcL$, we can define the frequency responses on the input signal $\x$ as
$\whz(\lambda_l)=\whx(\lambda_l)\whg(\lambda_l)$, or the inverse graph Fourier transform,
\begin{align}
z(i) = \sum_{l=1}^{N} \whx(\lambda_l)\whg(\lambda_l)\whu_l(i),
\end{align}
where $\whz(\lambda_l), \whx(\lambda_l), \whg(\lambda_l)$ are the Fourier coefficients corresponding to the spectrum $\lambda_l$. Hence, the matrix description of graph filtering is
\begin{align}
\z = \whg(\mcL)\x =\U\diag([\whg(\lambda_1), \cdots, \whg(\lambda_{N})])\U\tp\x.\label{eqn:Z}
\end{align}
We need to learn the filtering function $g(\cdot)$, but the computational cost of Eqn.~(\ref{eqn:L_SVD}) and Eqn.~(\ref{eqn:Z}) is expensive because the eigenvalue decomposition need to be done.

To address this problem, we may parameterize the filtering function $g(\cdot)$ with a polynomial approximation. As used in the literature~\cite{defferrard2016convolutional}, we employ the Chebyshev expansion of $K$ order~\cite{hammond2011wavelets} by defining the recurrent relation $T_k(x)=2xT_{k-1}(x)-T_{k-2}(x)$ with $T_0=1$ and $T_1=x$. Any one function in the space $x\in[-1,1]$ can be expressed with the expansion: $f(x)=\sum_{k=0}^\infty a_kT_k(x)$.
Suppose we take the $K$-order approximation for $g(\cdot)$, \ie,
\begin{align}
\whg(\lambda_l) = \sum_{k=0}^{K-1} \theta_{k}T_k(\widetilde{\lambda_l}), \label{eqn:g_poly}
\end{align}
where $\theta=[\theta_1,\cdots,\theta_K]\tp\in\mbR^K$ is the parameter vector of polynomial coefficients, and $\widetilde{\lambda_l}=\frac{2}{\lambda_{max}}\lambda_l-1$ with $\lambda_{max}=\max\{\lambda_{1},\cdots,\lambda_{N}\}\leq 2$. By substituting Eqn.~(\ref{eqn:g_poly}) into Eqn.~(\ref{eqn:Z}), we can derive the following equation,
\begin{align}
\z  = \sum_{k=0}^{K-1}\theta_k T_k(\frac{2}{\lambda_{max}}\mcL-\I)\x,\label{eqn:z_poly}
\end{align}
where we use this basic equation,
\begin{align}
\mcL^k&=\U\diag([\lambda_1^k,\cdots,\lambda_N^k])\U\tp\nonumber\\
&=(\U\diag([\lambda_1,\cdots,\lambda_N])\U\tp)^k.
\end{align}

Let $\widetilde{\mcL}=\frac{2}{\lambda_{max}}\mcL-\I$, if we denote the $K$ filter bases about the Laplacian matrix as $\{T_0(\widetilde{\mcL}), T_1(\widetilde{\mcL}), \cdots, T_{K-1}(\widetilde{\mcL})\}$, the final local spectral graph filtering can be written as
\begin{align}
\Z = [T_0(\widetilde{\mcL})\X, T_1(\widetilde{\mcL})\X, \cdots, T_{K-1}(\widetilde{\mcL})\X]\Theta, \label{eqn:Z_out}
\end{align}
where $\Theta\in\mbR^{3K\times d_z}$ is the parameters to be learnt with the $d_z$ output channels, and $\Z\in\mbR^{N\times d_z}$ is the responses of spectral graph filtering. As $\widetilde{\mcL}^k$ encodes a $k$-hop local neighborhood of each node, so the $K$-order polynomial in Eqn.~(\ref{eqn:z_poly}) is a exactly $K$-localized filter function of graphs. Correspondingly, $\Theta$ is the $K$-localized filtering parameter need to be solved.

\subsection{Detecting Salient Action Units}\label{sec:detect_au}

As observed that an action often occurs at a special body part, the detection of salient action units is necessary to reduce the disturbances of irrelevant joints as well as verify human cognition on actions. To detect action units, we propose a new layer named as action-attending layer to adaptively weight skeletal joints for different human actions, inspired by the attention mechanism used for various tasks, \eg, machine translation~\cite{bahdanau2014neural}, speech recognition~\cite{chorowski2015attention}, and image captioning~\cite{xu2015show}, etc. Our purpose is to decide what action units identify (or play a key role in) a special human action.

We stack this layer after the spectral graph filtering layer to take advantage of high-level features. As the $K$-order filtering has a $K$-hop receptive field as introduced in Section~\ref{sec:gcnn}, the filtering responses of each node actually assemble certain information of $K$-path neighbors around the node. Suppose the feature responses after spectral filtering as $\Z\in\mbR^{N\times d_z}$ (in Eqn.~\ref{eqn:Z_out}), we attempt to learn a projecting matrix to weight nodes. But considering different action cases, we expect that the projecting matrix can dynamically change with the different input $\Z$. That is, the dynamic matrix $\mcW\in\mbR^{N'\times N}$ is actually a parameterized function of the variable $\Z$, formally,
\begin{align}
    \tbZ &= \mcW(\Z)\Z,\label{eqn:tbZ}\\
    \st, & \mcW_{ij}\geq0,\quad \sum_{k=1}^{N}\mcW_{ik} = 1,\\
    & i=1,\cdots, N, \quad j=1,\cdots,N'. \nonumber
\end{align}
The larger the matrix item $\mcW_{ij}$ is, the more important the $j$-the node is for action recognition. To model the dynamic property, we define the dynamic function as
\begin{align}
    \mcW(\Z) &= (\tanh(\Z\Q + \b\tp)\V)\tp, \label{eqn:mcW}\\
    \mcW_{ij} &= \frac{\exp(\mcW_{ij})}{\sum_{k=1}^{N} \exp(\mcW_{ik})},
\end{align}
where $\Q\in\mbR^{d_z\times d'}, \V\in\mbR^{d'\times N'},\b\in\mbR^{d'}$ are the parameters to be solved. Hence, the dynamic function $\mcW$ takes advantage of the input feature $\Z$.

In addition to the detection of salient action units, the dynamic weighting function has two extra advantages:

(1) Order-independency of nodes.

Suppose all nodes are ordered in $\{v_1,\cdots, v_N\}$ and the signal matrix $\X=[\x_1,\cdots,\x_N]\tp$ is built, then we can extract the filtering feature $\Z=[\z_1,\cdots,\z_N]\tp$ in the same order of nodes, \ie, $\z_i$ is still associated with $v_i$. However, if all nodes are disordered, \eg, $\{v_N,v_{N-1}\cdots, v_1\}$, correspondingly, we have $\X'=[\x_N,\x_{N-1}\cdots,\x_1]\tp$ and $\Z'=[\z_{N},z_{N-1}\cdots,\z_1]\tp$. If we employ a constant projection $\W$ rather than the dynamic function $\mcW$, it will result into $\W\Z\neq \W\Z'$. That means, different traversing ways on a graph will produce different responses, which is unfeasible to feature comparisons, \eg, the feature $\tbZ$ will be flatten as the input to fed into recurrent neural network (See Section~\ref{sec:lstm}). However, the dynamic weighting function $\mcW$ is order-independent for graphical nodes, according to the following theory.
\begin{thm}
Given the dynamic function $\mcW$ in Eqn.~(\ref{eqn:mcW}), the output $\tbZ$ in Eqn.~(\ref{eqn:tbZ}) is irrelevant to the order of traversing order of graphical nodes.
\end{thm}
\begin{proof}
Suppose the input signal matrix $\X$, correspondingly, we can obtain the convolutional filtering feature $\Z$ and $\tbZ$ according to Eqn.~(\ref{eqn:Z_out}) and Eqn.~(\ref{eqn:tbZ}) respectively. Given any one permutation matrix $\mcP\in\{0,1\}^{N\times N}$, $\mcP\X$ equals to reorder all signals in $\X$. Let denote the feature $\Z'$ and $\tbZ'$ when taking $\mcP\X$ as the input. According to Eqn.~(\ref{eqn:Z_out}), we can have
\begin{align}
\Z' &= [T_0(\mcP\widetilde{\mcL}\mcP\tp)\mcP\X, \cdots, T_{K-1}(\mcP\widetilde{\mcL}\mcP\tp)\mcP\X]\Theta\nonumber\\
    &= [\mcP T_0(\widetilde{\mcL})\mcP\tp\mcP\X, \cdots, \mcP T_{K-1}(\widetilde{\mcL})\mcP\tp\mcP\X]\Theta\nonumber\\
    &= \mcP[T_0(\widetilde{\mcL})\X, \cdots, T_{K-1}(\widetilde{\mcL})\X]\Theta\nonumber\\
    &= \mcP\Z.
\end{align}
Note that the Laplacian matrix is also reordered by the permutation matrix $\mcP$. Now we only need to prove that $\tbZ=\tbZ'$. According to Eqn.~(\ref{eqn:tbZ}) and Eqn.~(\ref{eqn:mcW}), we have
\begin{align}
\tbZ' &=  \mcW(\Z')\Z' = \mcW(\mcP\Z)\mcP\Z \nonumber\\
      &= (\tanh(\mcP\Z\Q+\b\tp)\V)\tp\mcP\Z \nonumber\\
      &= (\mcP\tanh(\Z\Q+\b\tp)\V)\tp\mcP\Z \nonumber\\
      &= (\tanh(\Z\Q+\b\tp)\V)\tp\mcP\tp\mcP\Z \nonumber\\
      &= (\tanh(\Z\Q+\b\tp)\V)\tp\Z \nonumber\\
      &= \tbZ.
\end{align}
\end{proof}

(2) Dynamic pooling of nodes.

The dynamic function $\mcW$ may be regarded as a pooling operation on nodes. Given a row $\mcW_{i\cdot}$ of $\mcW$, the output $\mcW_{i\cdot}\Z$ is a new signal by weighting and combining nodes. Correspondingly, we can update the adjacency matrix of nodes and the Laplacian matrix,
\begin{align}
    \A' &= \mcW\A\mcW \tp,\\
    \widetilde{\mcL}' &= \I-\D^{-1/2}\A' \D^{-1/2},
\end{align}
where $\A',\widetilde{\mcL}'\in\mbR^{N'\times N'}$. Thus, the new Laplacain matrix $\widetilde{\mcL}'$ and the feature $\tbZ$ can be fed into the next layer and make the neural network go deeper.

\subsection{Modeling Temporal Motions}\label{sec:lstm}

After extracting the spatial graphical features, we need to model temporal motion variations of a skeletal sequence. Many non-linear dynamic models can be used to solve this problem. Here we employ a special class of recurrent neural networks (RNN), long short-term memory (LSTM)~\cite{hochreiter1997long}, which can mitigate gradient vanishment when back-propagating gradients. LSTM has demonstrated the powerful ability to model long-range dependencies~\cite{graves2013generating,srivastava2015unsupervised,sutskever2014sequence}. Suppose the graphical feature of the $t$-th skeletal frame is $\tbz_t=\text{vectorize}(\tbZ_t)\in \mbR^{N'd_z}$, the cell output $\h_t\in \mbR^{d_h}$ and states $\c_t\in\mbR^{d_h}$ are intermediate vectors, formally, the motion variations can be modeled as
\begin{align}\label{equ:lstm}
    \i &= \sigma(\W_{zi}\tbz_t + \W_{hi}\h_{t-1} + \w_{ci}\odot \c_{t-1} + \b_i),\\
    \f &= \sigma(\W_{zf}\tbz_t + \W_{hf}\h_{t-1} + \w_{cf}\odot \c_{t-1} + \b_f),\\
    \c_t &= \f_t\odot \c_{t-1} + \i_t\odot \tanh(\W_{zc}\tbz_t + \W_{hc}\h_{t-1} + \b_c),\\
    \o_t &= \sigma(\W_{zo}\tbz_t + \W_{ho}\h_{t-1} + \w_{co}\odot \c_t + \b_o),\\
    \h_t &= \o\odot\tanh(\c_t),
\end{align}
where $\sigma(\cdot)$ is the elementwise sigmoid function, \ie, $\sigma(\x)=1/(1+e^{-\x})$, $\odot$ denotes the Hadamard product and $\i, \f, \o, \c\in \mbR^{d_h}$ are respectively the {\em input gate}, {\em forget gate}, {\em output gate}, {\em cell} and {\em cell input} activation vectors. The weight matrices \{$\W_{z\cdot}\in \mbR^{d_h\times N'd_z}, \W_{h\cdot}\in \mbR^{d_h \times d_h}, \w_{c\cdot}\in \mbR^{d_h}$\} and the bias vectors \{$\b_i,\b_f,\b_c,\b_o\in\mbR^{d_h}$\} are the model parameters to be solved. Finally, we use the output gate $\o_t$ as the response at the $t$-th time slice.

\section{Implementation Details}\label{sec:implementation}

In this section, we will introduce our implementation details including network architecture and data augmentation.

\subsection{Network Architecture}

For the undirected graph, we simply construct edge connections based on human bones. That means, if two joints are bridged with a bone, the edge weight is assigned to 1, otherwise 0. Our plain network contains two spectral graph filtering layers along with two action-attending layers followed by a temporal recurrent layer, as shown in Fig.~(\ref{fig:architecture}). In the spectral filtering layers, the receptive fields are set to $K=10$ as default, and the output signals have the length of 32 dimensions and 64 dimensions respectively for two layers. In the action-attending layer, we take a simple design rule: the dimensions remain invariant with regard to the input, \ie, $N'=N, d'=d_z$. In the temporal recurrent layer, we employ the classic LSTM unit to model the temporal dynamics, where the dimension of hidden units is set to 256. In the full connected layer, the output has the same dimension to the input. The network ends with a cross-entropy loss used for classification. The learning rate of network is set to 0.02 with a momentum of 0.9. More analysis/discussion of parameters can be found in Section~\ref{sec:params_analy}. The concrete implementation takes TensorFlow as the infrastructure.

\subsection{Data Processing}

As skeletal data is usually captured from multi-view points and human actions are independent on the user coordinate system, we modify the origin of the coordinate system as the orthocenter of joints for each frame of skeleton, \ie,
$\mcO=\frac{1}{N}\sum_{i=1}^{N} \x_i$, where $\x_i\in \mbR^3$ is a 3D coordinate of the $i$-th joint, $N$ is the number of joints.

To enhance the robustness of model training, we perform data augmentation as widely used in previous deep learning literature~\cite{liu2016spatio,shahroudy2016ntu}. Concretely, for each action sequence, we split the sequence into several equal sized subsequences, here 12 segments, and then pick one frame from each segment randomly to generate a large amount of training sequences. In addition, we randomly scale the skeletons by multiplying a factor in [0.98, 1.02] for the sake of the adaptive capability of scaling.

\section{Experiments}\label{sec:experiments}

To evaluate our proposed A$^2$GNN, we conduct extensive experiments on four benchmark skeleton-based action datasets, including HDM05~\cite{muller2007documentation}, Florence 3D~\cite{seidenari2013recognizing}, Large Scale Combined dataset (LSC)~\cite{zhang2016} and NTU RGB+D~\cite{shahroudy2016ntu}. A brief summarization about them is given in Table~\ref{tab:datasets}. Below we will compare our A$^2$GNN with the recent state-of-the-art methods, then analyze confusion matrices and action unit detection, and finally discuss some network parameters.

\begin{table}[t]
  \caption{Summarization of four action recognition datasets.}
  \label{tab:datasets}
  \centering
  \renewcommand\arraystretch{1.2}
  \begin{tabular}{lcccc}
    \hline\hline
    Dataset&\# Joints&\# Actions&\# Subjects&\# Sequences \\ \hline
    HDM05~\cite{muller2007documentation}        & 31    & 130   & 5     & 2337 \\
    Florence 3D~\cite{seidenari2013recognizing} & 15    & 9     & 10    & 215 \\
    LSC~\cite{zhang2016}                        & 15/20 & 88    & 79    & 3898 \\
    NTU RGB+D~\cite{shahroudy2016ntu}           & 25    & 60    & 40    & 56880 \\
    \hline\hline
  \end{tabular}
\end{table}

\subsection{Datasets}

\subsubsection{HDM05~\cite{muller2007documentation}}

This dataset was captured by using an optical marker-based Vicon system, and gathered 2337 action sequences for 130 motion classes, which are performed by 5 non-professional actors named ``bd", ``bk", ``dg", ``mm" and ``tr". Each skeleton data is represented with 31 joints. Until now, this dataset should involve the most skeleton-based action categories to the best of our knowledge. Due to the intra-class variations and large number of motion classes, this dataset is challenging in action recognition.

\subsubsection{Florence 3D~\cite{seidenari2013recognizing}}

This dataset was collected via a stationary Microsoft Kinect camera. It consists of 215 action sequences from 10 subjects for 9 actions: { wave, drink from a bottle, answer phone, clap, tight lace, sit down, stand up, read watch, bow}. Only 15 joints are recorded for each skeletal data. As a few skeletal joints, some types of actions are difficult to distinguish, such as drink from a bottle, answer phone and read watch.

\subsubsection{LSC~\cite{zhang2016}}

This dataset combines nine publicly available datasets, including MSR Action3D Ext~\cite{li2010action,wang2015convnets,wang2016action}, UTKinect-Action3D~\cite{xia2012view}, MSR DailyActivity 3D~\cite{wang2012mining}, MSR Action Pairs 3D~\cite{oreifej2013hon4d}, CAD120~\cite{koppula2013learning}, CAD60~\cite{sung2012unstructured}, G3D~\cite{bloom2013dynamic,bloom2012g3d}, RGBD-HuDa~\cite{ni2011rgbd}, UTD-MHAD~\cite{chen2015utd}, and form a complex action dataset with 94 actions. As some samples have not skeleton information, we remove them and construct a skeleton dataset of 88 actions by following previous standard protocols. As each individual dataset has its own characteristics in the action execution manners, backgrounds, acting positions, view angles, resolutions, and sensor types, the combination of a large number of action classes makes the dataset more challenging in suffering large intra-class variation compared to each individual dataset.

\subsubsection{NTU RGB+D~\cite{shahroudy2016ntu}}

This dataset is collected by Microsoft Kinect v2 cameras from different views. It consists of 56880 sequences and over 4 million frames for 60 distinct actions, including various of daily actions and pair actions. These actions were performed by 40 subjects aged between 10 and 35. The skeleton data is represented by 25 joints. As far as we know, this dataset is currently the largest skeleton-based action recognition dataset. The large intra-class and view point variations make this dataset great challenging. Meanwhile, a large amount of samples will bring a new challenging to the current skeleton-based action recognition methods.

\subsection{Comparisons with State-of-the-art Methods}

\begin{table}[t]
  \caption{Comparisons on HDM05 dataset. }
  \label{tab:HDM05}
  \centering
  \renewcommand\arraystretch{1.2}
  \begin{tabular}{l c c}
    \hline\hline
    Method  &  \tabincell{c}{Protocol~\cite{wang2015beyond}\\Accuracy} & \tabincell{c}{Protocol~\cite{huang2017riemannian}\\Accuracy}\\ \hline
    RSR-ML~\cite{harandi2014manifold}       & 40.0\%        & - \\
    Cov-RP~\cite{tuzel2006region}           & 58.9\%        & - \\
    Ker-RP~\cite{wang2015beyond}            & 66.2\%        & - \\
    SPDNet~\cite{huang2017riemannian}       & -             & 61.45\%$\pm$1.12\\
    Lie Group~\cite{vemulapalli2014human}   & -             & 70.26\%$\pm$2.89\\
    LieNet~\cite{huang2016deep}             & -             & 75.78\%$\pm$2.26\\
    {P-LSTM~\cite{shahroudy2016ntu}}        & 70.4\%        & 73.42\%$\pm$2.05 \\
    {\bf A$^2$GNN}                          & {\bf 76.5\%}  & {\bf84.47\%$\pm$1.52} \\
    \hline\hline
  \end{tabular}
  \par
  \vskip 0.1 cm
  \leftline{\quad\quad\quad\quad\quad *Note that all 130 classes are used here.}
\end{table}

\begin{table}[t]
  \caption{Comparisons on Florence dataset.}
  \label{tab:Florence}
  \centering
  \renewcommand\arraystretch{1.2}
  \begin{tabular}{lc}
    \hline\hline
    Method  &  Accuracy \\ \hline
    Multi-part Bag-of-Poses~\cite{seidenari2013recognizing} & 82.00\% \\
    Riemannian Manifold~\cite{devanne20153}                 & 87.04\% \\
    Lie Group~\cite{vemulapalli2014human}                   & 90.88\% \\
    Graph-Based~\cite{wang2016graph}                        & 91.63\% \\
    MIMTL~\cite{yang2017discriminative}                     & 95.29\% \\
    {P-LSTM~\cite{shahroudy2016ntu}}                        & 95.35\% \\
    {\bf A$^2$GNN} & {\bf 98.60\%} \\
    \hline\hline
  \end{tabular}
\end{table}

\begin{table}[t]
  \caption{Comparisons on Large Scale Combined dataset.}
  \centering
  \renewcommand\arraystretch{1.2}
  \label{tab:Combined}
    \begin{tabular}{lcccc}
    \hline\hline
    \multirow{2}{*}{Method} &
    \multicolumn{2}{c}{Cross Sample} & \multicolumn{2}{c}{Cross Subject} \cr\cline{2-5}
    &Precision&Recall&Precision&Recall\cr
    \hline
    HON4D~\cite{oreifej2013hon4d}         & 84.6\% & 84.1\% & 63.1\% & 59.3\% \\
    Dynamic Skeletons~\cite{hu2015jointly}& 85.9\% & 85.6\% & 74.5\% & 73.7\% \\
    {P-LSTM~\cite{shahroudy2016ntu}}      & 84.2\% & 84.9\% & 76.3\% & 74.6\% \\
    {\bf A$^2$GNN}          &{\bf 87.6\%}&{\bf 88.1\%}&{\bf 84.0\%}&{\bf 82.0\%}\\
    \hline\hline
    \end{tabular}
\end{table}

\begin{table}[t]
  \caption{Comparisons on NTU RGB+D dataset.}
  \label{tab:NTU}
  \centering
  \renewcommand\arraystretch{1.2}
  \begin{tabular}{lcc}
    \hline\hline
    Method & \tabincell{c}{Cross Subject\\Accuracy} & \tabincell{c}{Cross View\\Accuracy} \\
    \hline
    HON4D~\cite{oreifej2013hon4d}                   & 30.56\% &  7.26\% \\
    Lie Group~\cite{vemulapalli2014human}           & 50.08\% & 52.76\% \\
    Skeletal Quads~\cite{evangelidis2014skeletal}   & 38.62\% & 41.36\% \\
    Dynamic Skeletons~\cite{hu2015jointly}          & 60.23\% & 65.22\% \\
    HBRNN~\cite{batabyal2015action}                 & 59.07\% & 63.97\% \\
    LieNet~\cite{huang2016deep}                     & 61.37\% & 66.95\% \\
    Deep RNN~\cite{shahroudy2016ntu}                & 56.29\% & 64.09\% \\
    Deep LSTM~\cite{shahroudy2016ntu}               & 60.69\% & 67.29\% \\
    P-LSTM~\cite{shahroudy2016ntu}                  & 62.93\% & 70.27\% \\
    ST-LSTM~\cite{liu2016spatio}                    & 69.2\%  & 77.7\% \\
    {\bf A$^2$GNN}                      & {\bf 72.74\%} & {\bf 82.80\%} \\
    \hline\hline
  \end{tabular}
\end{table}

\begin{figure}[t]
  \centering
  \subfigure[]{
    \label{fig:HDM05:All}
    \includegraphics[width=3.4in]{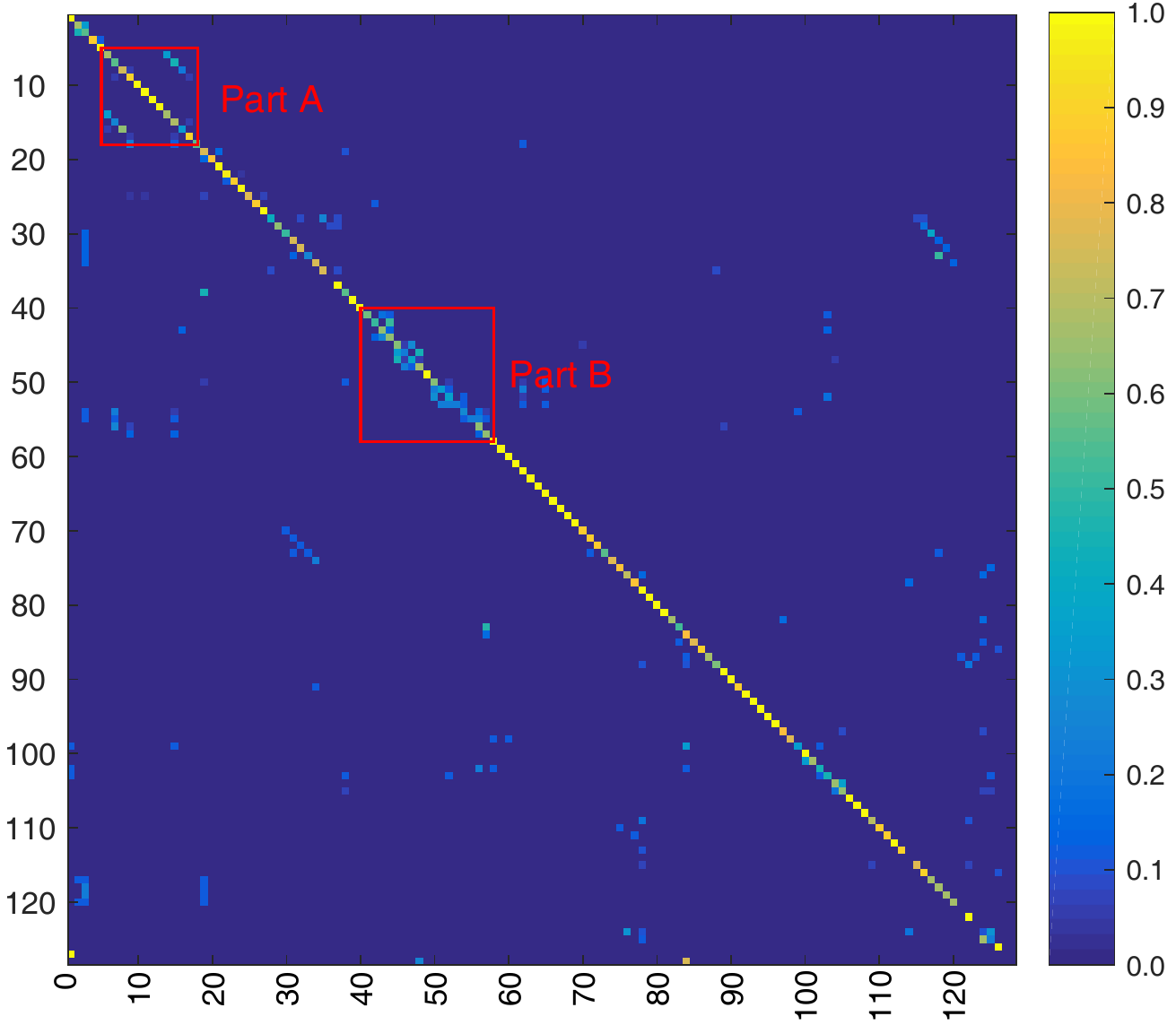}}
  \hspace{0.1in}
  \subfigure[]{
    \label{fig:HDM05:A}
    \includegraphics[width=1.6in]{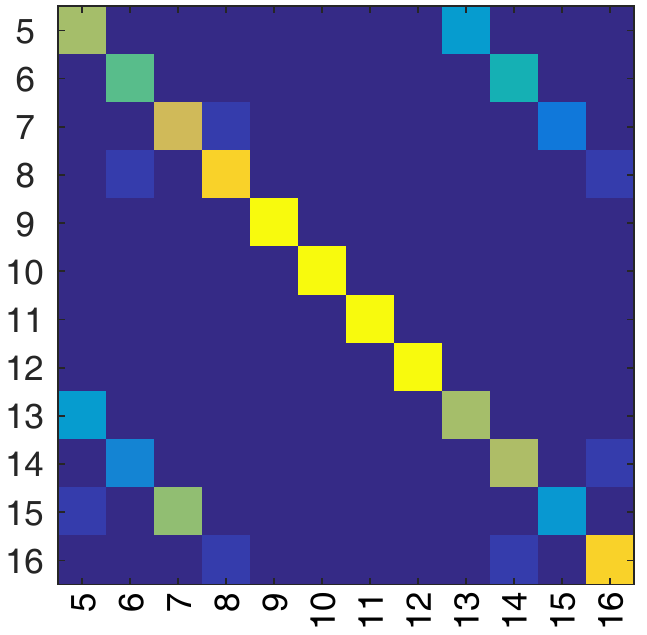}}
  \hspace{-0.05in}
  \subfigure[]{
    \label{fig:HDM05:B}
    \includegraphics[width=1.6in]{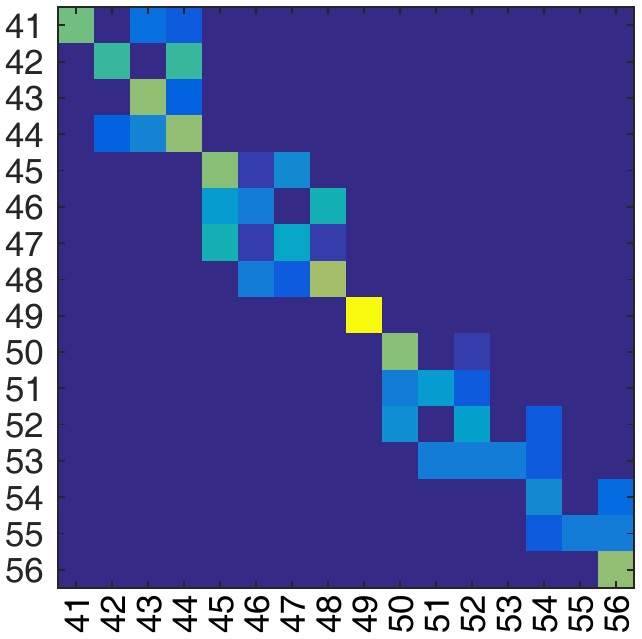}}
  \caption{Confusion matrix of our A$^2$GNN on HDM05 dataset according to the testing protocol in~\cite{huang2017riemannian}. (b) and (c) are the close up of Part A and Part B in (a). X-axis and Y-axis are associated with the indices of action classes.\protect\\
  \textit{**~Indices of part classes: 5-depositFloorR; 6-depositHighR; 7-depositLowR; 8-depositMiddleR; 13-grabFloorR; 14-grabHighR; 15-grabLowR; 16-grabMiddleR; 41-kickLFront1Reps; 42-kickLFront2Reps; 43-kickLSide1Reps; 44-kickLSide2Reps; 45-kickRFront1Reps; 46-kickRFront2Reps; 47-kickRSide1Reps; 48-kickRSide2Reps; 50-punchLFront1Reps; 51-punchLFront2Reps; 52-punchLSide1Reps; 53-punchLSide2Reps; 54-punchRFront1Reps; 55-punchRFront2Reps; 56-punchRSide1Reps.}}
  \label{fig:HDM05}
\end{figure}

\begin{figure}[t]
  \centering
  \subfigure[Cross View]{
    \label{fig:NTU:CV}
    \includegraphics[width=1.5in]{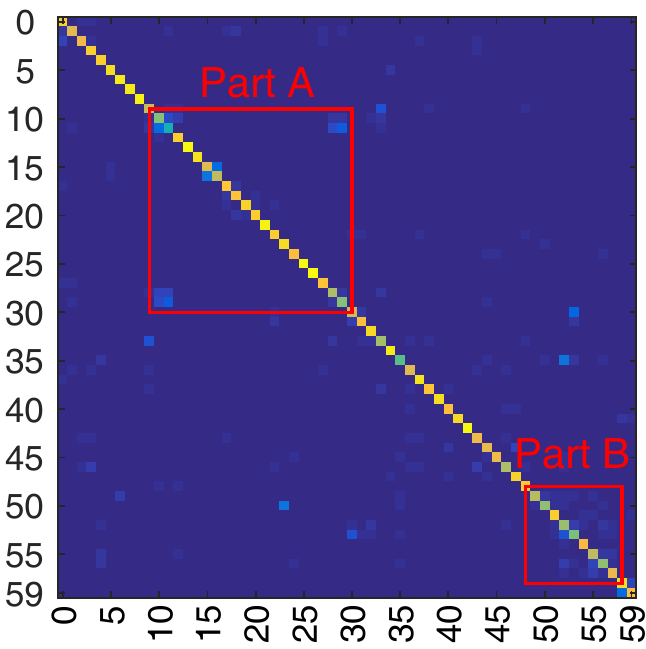}}
  \hspace{-0.1in}
  \subfigure[Cross Subject]{
    \label{fig:NTU:CS}
    \includegraphics[width=1.8in]{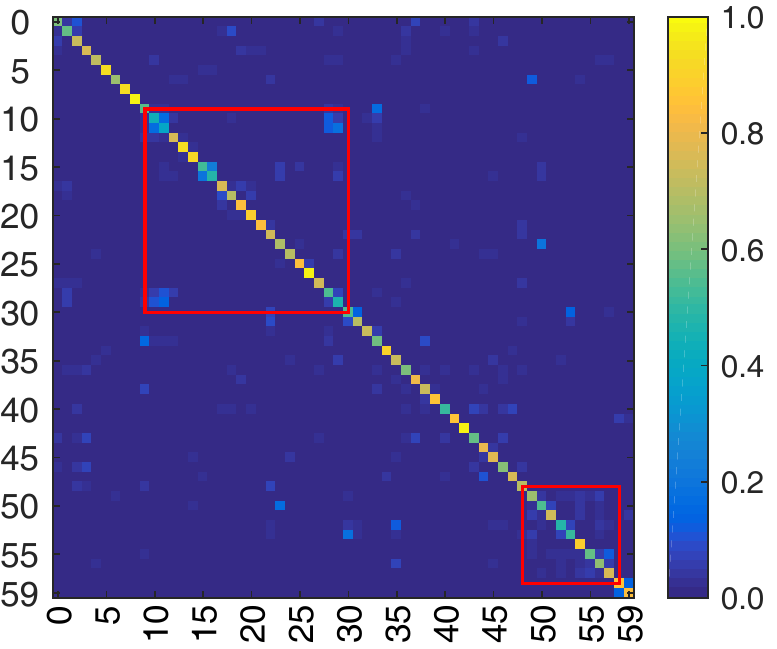}}
  \hspace{-0.1in}
  \subfigure[]{
    \label{fig:NTU:1}
    \includegraphics[width=1.4in]{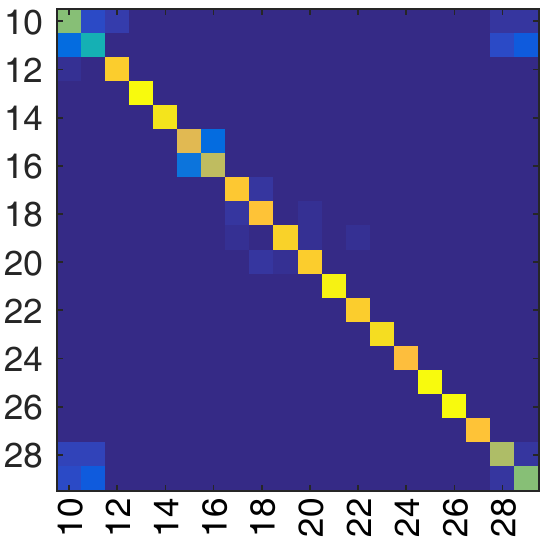}}
  \hspace{-0.1in}
  \subfigure[]{
    \label{fig:NTU:2}
    \includegraphics[width=1.4in]{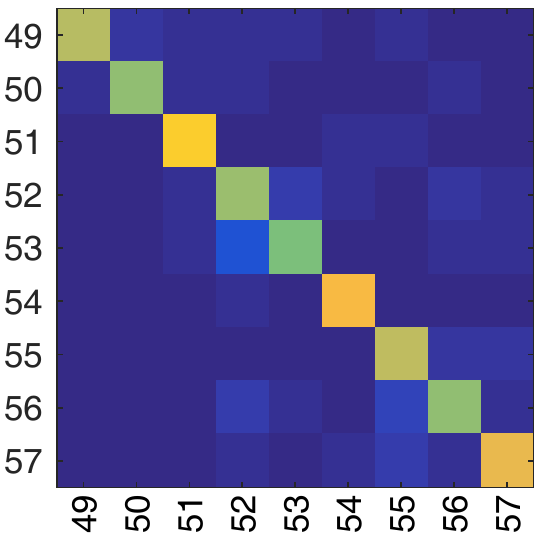}}
  \caption{Confusion matrix of our A$^2$GNN on NTU RGB+D dataset. (c) and (d) are the close up of Part A and Part B in (a). X-axis and Y-axis are associated with the indices of action classes.\protect\\
  \textit{**~Indices of part classes: 10-reading; 11-writing; 15-wear a shoe; 16-take off a shoe; 17-wear on glasses; 18-take off glasses; 19-put on a hat/cap; 20-take off a hat/cap; 28-playing with phone/tablet; 29-typing on a keyboardp; 49-punching/slapping other person; 50-kicking other person; 51-pushing other person; 52-pat on back of other person; 53-point finger at the other person; 54-hugging other person; 55-giving something to other person; 56-touch other person's pocket; 57-handshaking.}}
  \label{fig:NTU}
\end{figure}

\begin{figure}[t]
  \centering
  \includegraphics[width=3.3in]{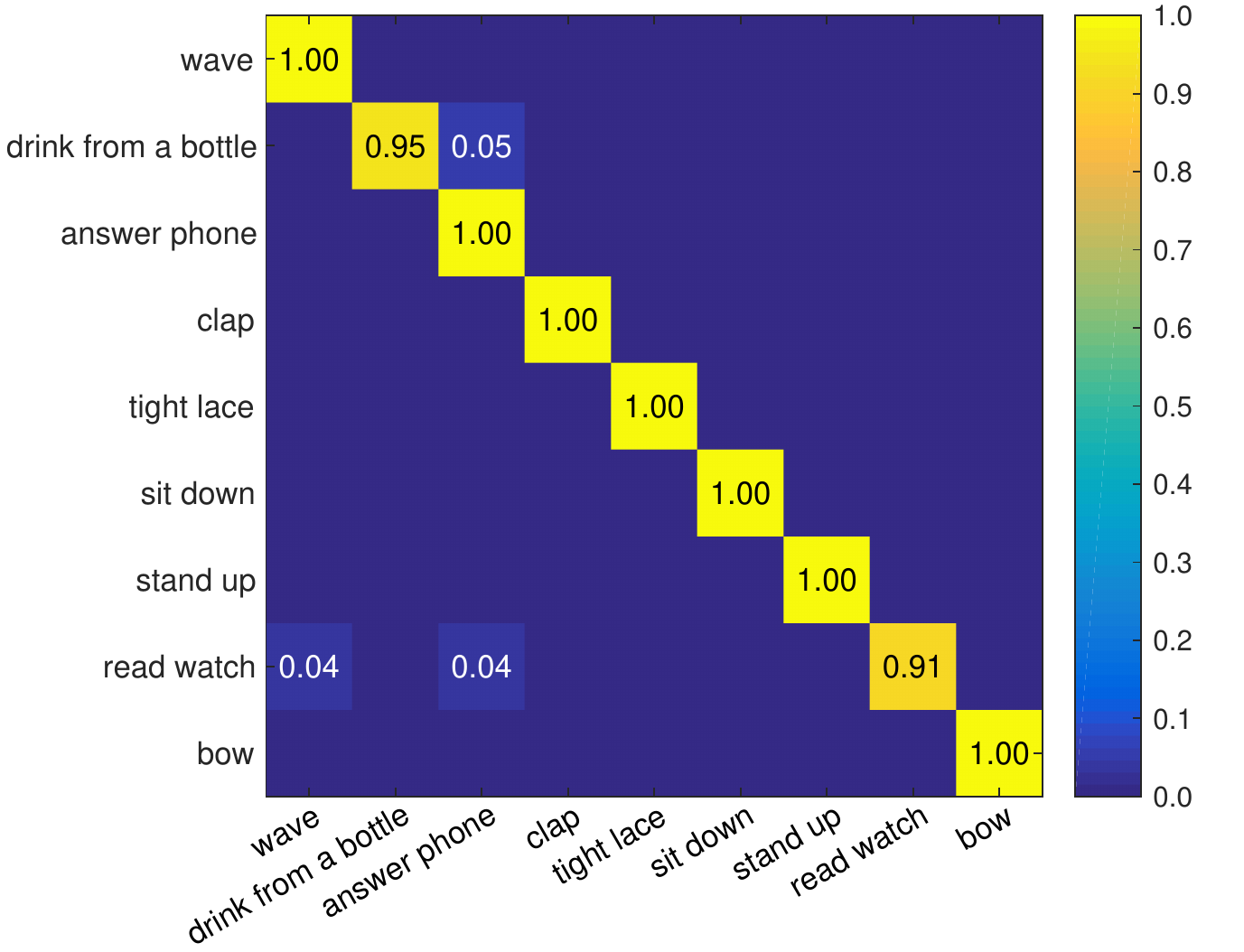}
  \caption{Confusion matrix of our A$^2$GNN on Florence 3D Actions dataset. }
  \label{fig:Florence}
\end{figure}

\subsubsection{The Results}~

\paragraph{HDM05} To compare with those previous literatures, we conduct two types of experiments by following two widely-used protocols. First, we follow the protocol used in~\cite{wang2015beyond} to perform action recognition on all of the 130 classes. Actions of two subjects named ``bd" and ``mm" are used to train the model, and the remaining three ones for testing. The comparison results are shown in the left column of Table~\ref{tab:HDM05}. Second, to fairly compare the current deep learning methods, we follow the settings of the literature~\cite{huang2017riemannian} to conduct 10 random evaluations, each of which randomly selects half of the sequences for training and the rest for testing. The results are reported in the right column of Table~\ref{tab:HDM05}.

\paragraph{Florence 3D} We follow the experimental settings of the literature~\cite{wang2016graph} to perform leave-one-subject-out cross-validation. In each round, skeletal data of 9 subjects is taken for training and the remain one for testing. The experimental results are reported in Table~\ref{tab:Florence}.

\paragraph{LSC} We follow the protocol designed in the recent work~\cite{zhang2016} to conduct two types of experiments, {random cross subject} and {random cross sample}. The experimental results are reported in Table~\ref{tab:Combined}.

\paragraph{NTU RGB+D} Different from the above three datasets, we preprocess the joint coordinates in a way similar to~\cite{shahroudy2016ntu}. Concretely, after translating the original coordinates of body joints as mentioned above, we rotate the $x$ axis parallel to the 3D vector from ``right shoulder" to ``left shoulder" and $y$ axis towards the 3D vector from ``spine base" to ``spine". The $z$ axis is fixed as the new $x\times y$. This dataset has two types of standard evaluation protocols~\cite{shahroudy2016ntu}. One is cross-subject evaluation, for which half of the subjects are used for training and the remaining ones for testing. The second is cross-view evaluation, for which two viewpoints are used for training and one is left out for testing. The experimental results are shown in Table~\ref{tab:NTU}.

\subsubsection{The Analysis}~

As shown in Table~\ref{tab:HDM05}$\sim$\ref{tab:NTU}, we compare the current state-of-the-art methods on different datasets, including the shallow learning methods and the deep learning methods. From these results, we have the following observations:
\begin{itemize}
  \item \textit{Matrix-based descriptors (\eg, covariance or its variants) are conventionally used to model spatial or temporal relationships}. Moreover, these descriptors are often regarded to be embedded on specific geometric manifolds~\cite{vemulapalli2014human,huang2016deep,huang2017riemannian}. The advanced variant~\cite{wang2015beyond} improved the performance by constructing some robust SPD matrices. More recently, SPDNet~\cite{huang2016deep} and LieNet~\cite{huang2017riemannian} attempted to learn deep features from those raw matrix descriptors under the assumption of manifold. Although the deep manifold learning strategy on matrix descriptors raises a promising direction to some extent, the matrix-based representations principally limit their capability of modeling dynamic variations, because the only second-order statistic relationship of skeletal joints is preserved in the descriptors, whereas first-order statistics is also informative~\cite{ranzato2010modeling}.
  \item \textit{Deep features are more effective than those shallow features for skeleton-based action representation}. The advanced nonlinear dynamic networks, specifically Deep-RNN, Deep-LSTM, P-LSTM~\cite{shahroudy2016ntu} and ST-LSTM~\cite{liu2016spatio}, largely improve the action recognition performance, due to the good encoding capability of gated network units. Most of them use recurrent networks to model temporal dynamics. Besides, ST-LSTM also attempted to model spatial skeletal joints by taking a tree-structure traversal way on spatial joints. Similar to them, we also use recurrent neural network to model temporal dynamics. But different from them, we directly extract high-level semantic features from spatial skeletal graphs like the standard convolutional neural network.
  \item \textit{The proposed deep graph method is superior to the recent graph-based method~\cite{wang2016graph}}. As shown in Table~\ref{tab:Florence}, our A$^2$GNN has a large improvement (about 7\%) in contrast to the work~\cite{wang2016graph}. In principle, our A$^2$GNN is very different from this work~\cite{wang2016graph}, although graph is used for both. First, the graph is used to describe skeleton at each temporal slice for our method, not those segmented motion parts (called motionlets). Second, the purpose of the use of graph is to extract skeletal features for ours, not model the relationship of motion parts. Third, our method is a fully end-to-end deep learning architecture, not abide by the conventional two-step way: i) construct graphs of motion parts, and ii) compute the similarity between graphs via subgraph-patten graph kernel. Besides, the action-attending mechanism is introduced into our network architecture.
  \item \textit{Our proposed A$^2$GNN greatly improves the current state-of-the-art on most datasets}. On the Florence dataset, our method achieves a nearly perfect performance 98.60\%. On the current largest dataset NTU RGB+D dataset, the state-of-the-art performance is pushed to the higher 72.74\% and 82.80\%, from 69.2\% and 77.7\% for ST-LSTM. In summary, we can benefit from the deep graph network architecture. The reason can be two folds: i) the deep skeletal graph features, rather than simple spatial or temporal features learnt by LSTM; and ii) the preservation of more original feature information (as signals of each node), rather than only the second-order statistics of skeletal joints.
  \item \textit{Different performances occur on different datasets}. Among the four datasets, Florence 3D is the simplest one with 215 sequences and 9 action classes, so most methods can obtain a good accuracy. The most difficult dataset should be the largest dataset NTU RGB+D dataset, which consists of 56880 sequences and covers various of daily actions and pair actions. The cross subject accuracy on it just surpasses 70\% due to various entangled actions as analyzed in Section~\ref{sec:confusion}. Specifically, the phenomenon of entangled actions deteriorates in HDM05, which contains many confusable classes, such as walk step number, walk start with left or right, etc.
  \item \textit{Cross subject is more difficult than cross view or cross sample}. The phenomenon is observed from Table~\ref{tab:Combined}, Table~\ref{tab:NTU}, and Table~\ref{tab:HDM05} (the left/right column w.r.t cross subject/cross sample). It is easy to understand, each subject has itself action characteristics. In the cross subject task, more unforeseeable information exists the testing set, compared to the other tasks.
\end{itemize}

\begin{figure*}[!t]
 \centering
  \subfigure[Drink from a bottle]{
    \label{fig:F_1}
    \includegraphics[width=3in]{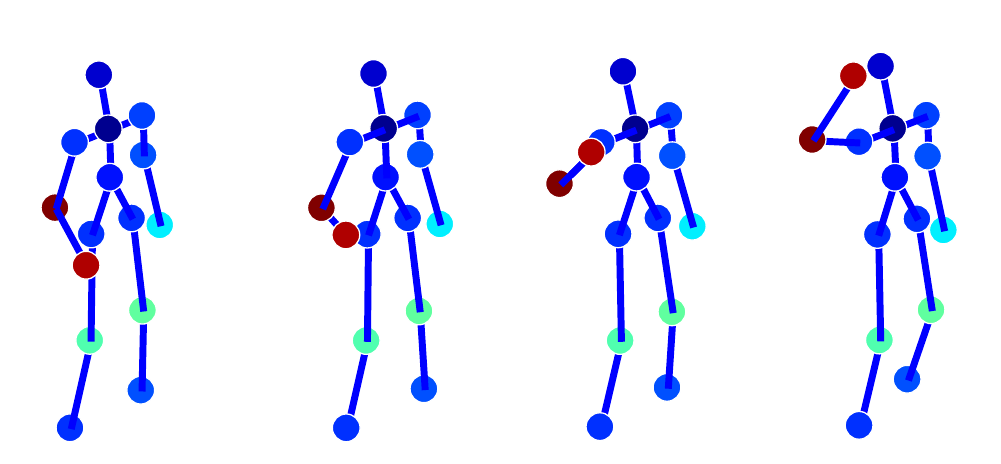}}
  \hspace{0.2in}
  \subfigure[Answer phone]{
    \label{fig:F_2}
    \includegraphics[width=3in]{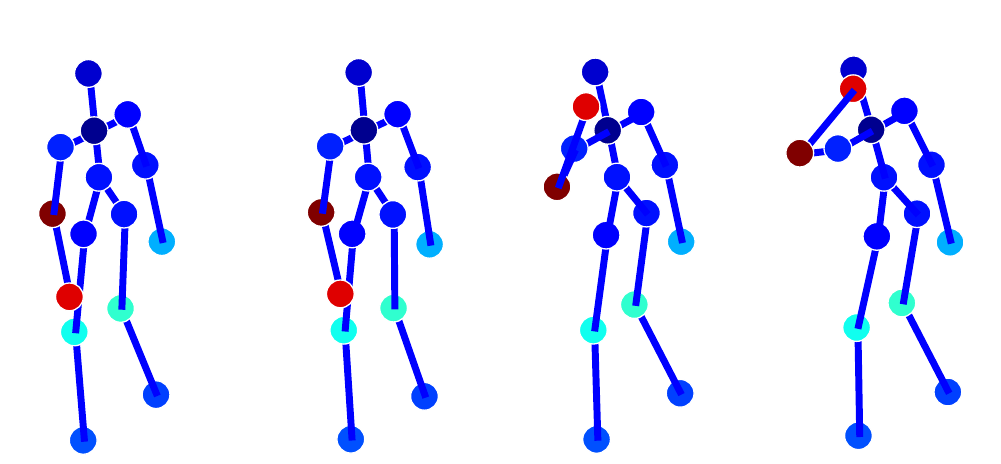}}
  \subfigure[Clap]{
    \label{fig:F_3}
    \includegraphics[width=3in]{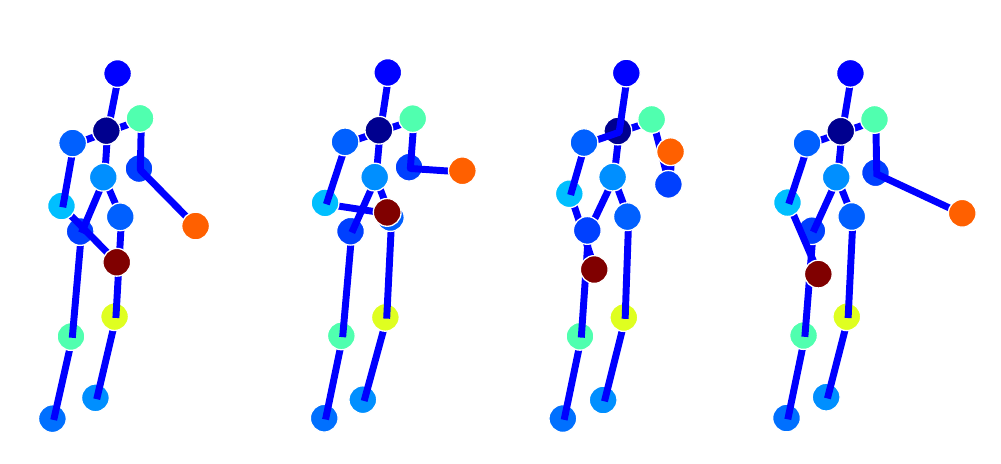}}
  \hspace{0.2in}
  \subfigure[Tight lace]{
    \label{fig:F_4}
    \includegraphics[width=3in]{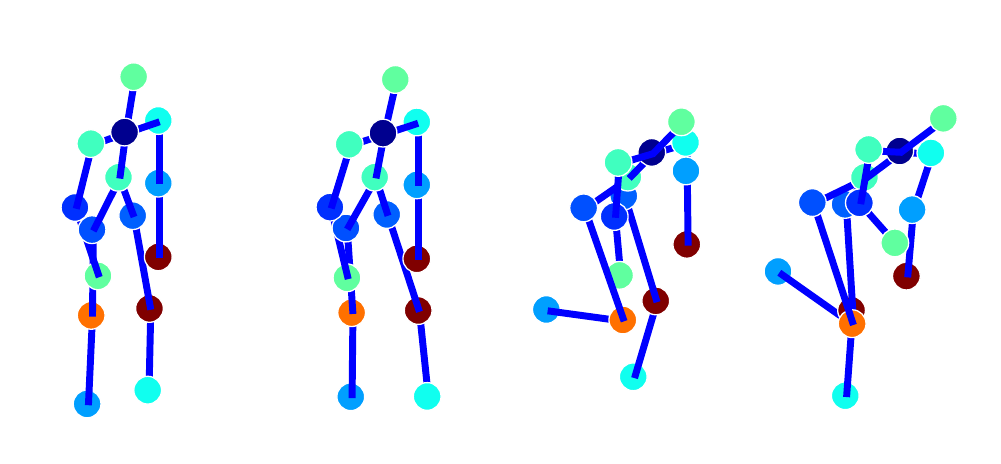}}
  \subfigure[Sit down]{
    \label{fig:F_5}
    \includegraphics[width=3in]{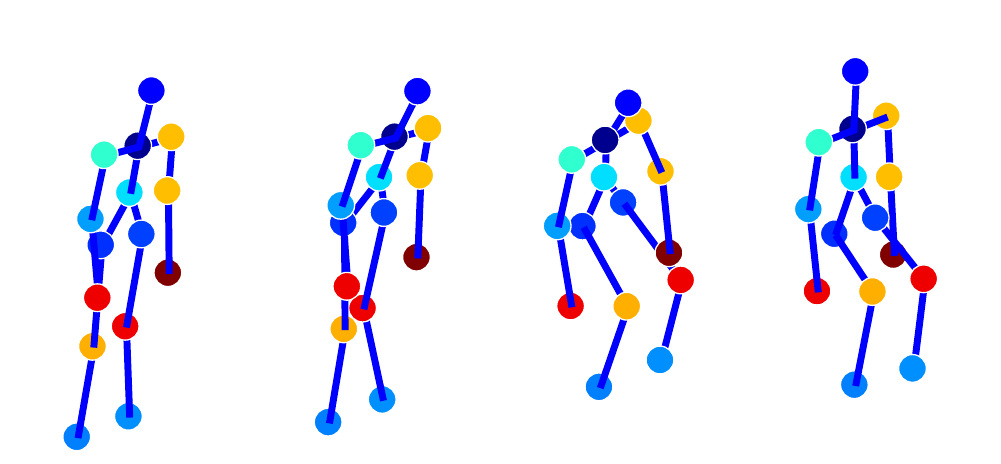}}
  \hspace{0.2in}
  \subfigure[Stand up]{
    \label{fig:F_7}
    \includegraphics[width=3in]{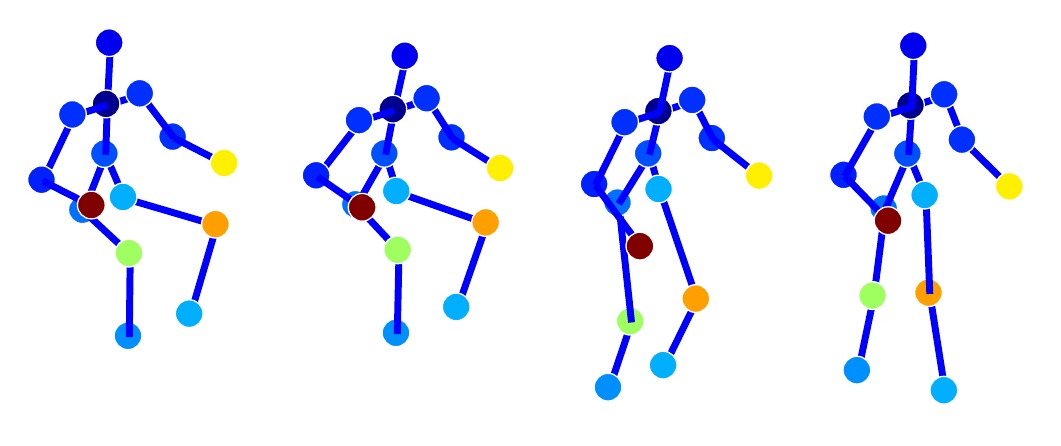}}
  \subfigure[Read watch]{
    \label{fig:F_6}
    \includegraphics[width=3in]{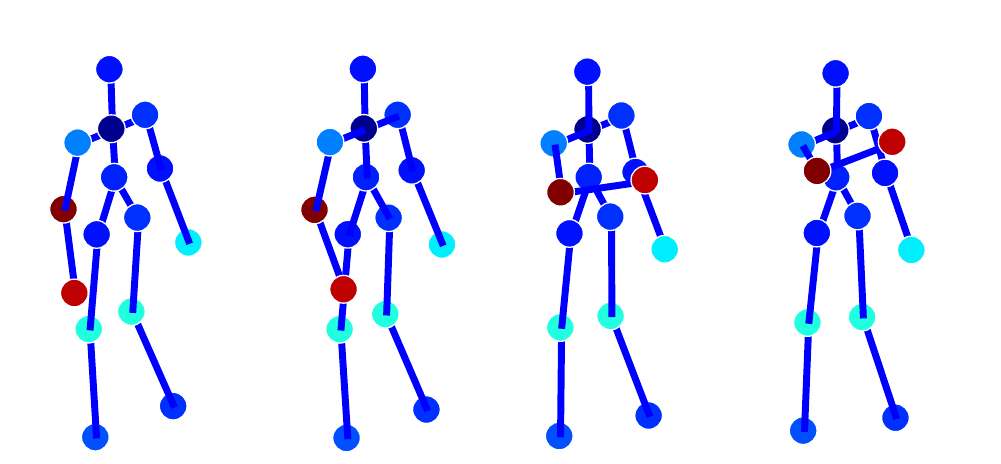}}
  \hspace{0.2in}
  \subfigure[Bow]{
    \label{fig:F_8}
    \includegraphics[width=3in]{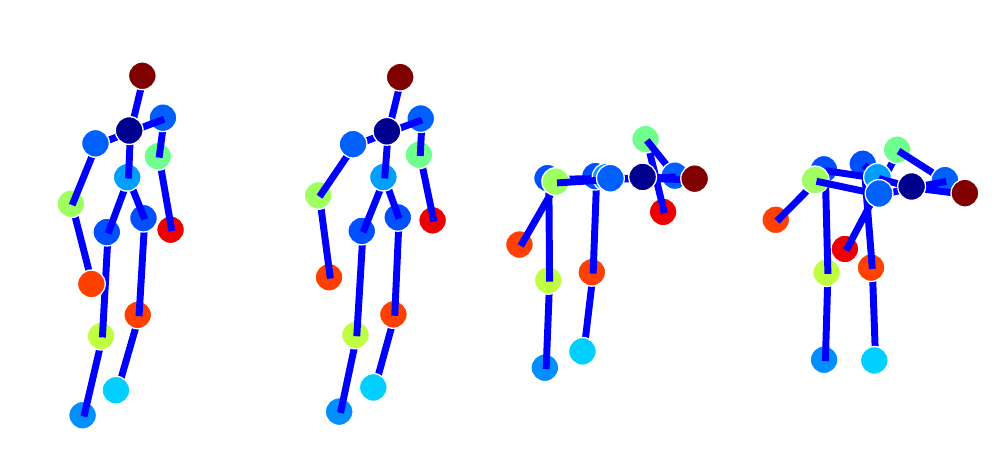}}
  \subfigure[Wave]{
    \label{fig:F_9}
    \includegraphics[width=3in]{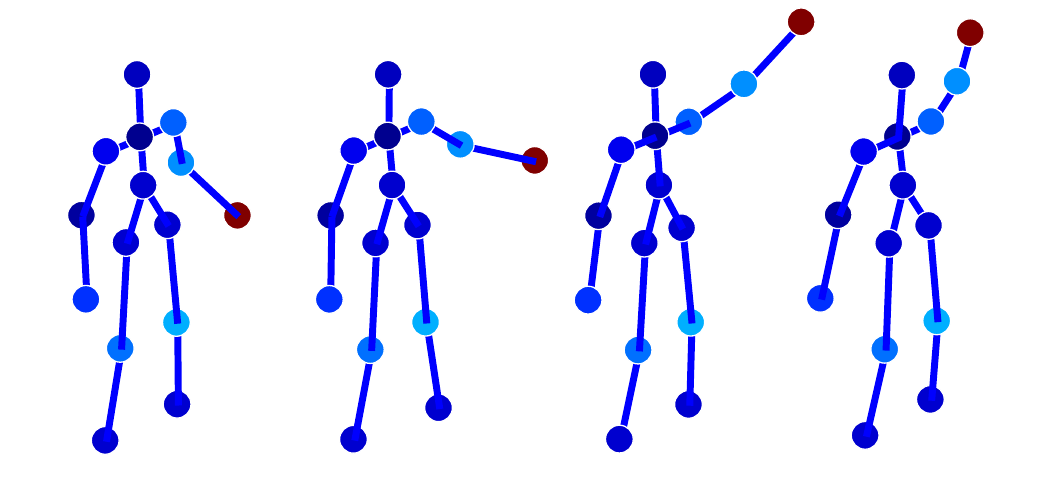}}
  \hspace{0.2in}
  \subfigure[Wave]{
    \label{fig:F_10}
    \includegraphics[width=3in]{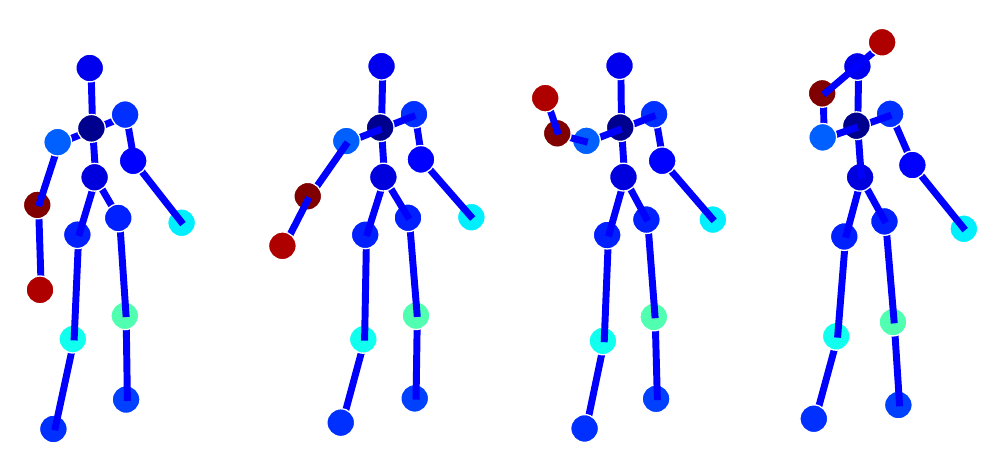}}
  \subfigure[Colorbar]{
    \includegraphics[width=2.5in]{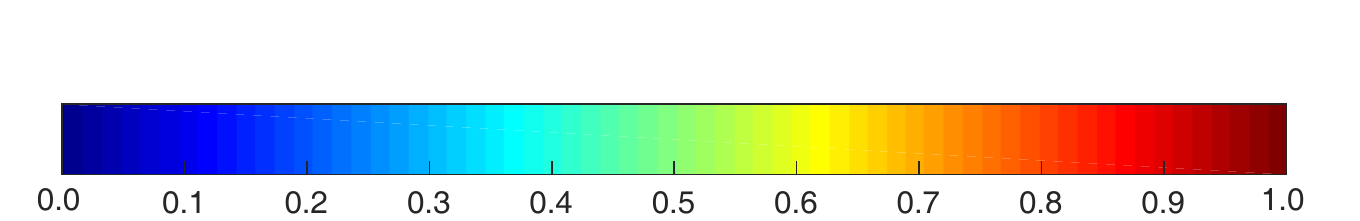}}
  \caption{Visual examples of detected salient action units of all 9 different action classes on Florence 3D. Higher weights in the colorbar means more important to characterize the action. Note that the motion sequences in (i) and (j) are annotated as one class (``wave") in this dataset, although one is left arm while the other is right arm.}
  \label{fig:visual:Florence}
\end{figure*}

\begin{figure*}[!t]
  \centering
  \subfigure[Drink water]{
    \label{fig:N_00}
    \includegraphics[width=2.8in]{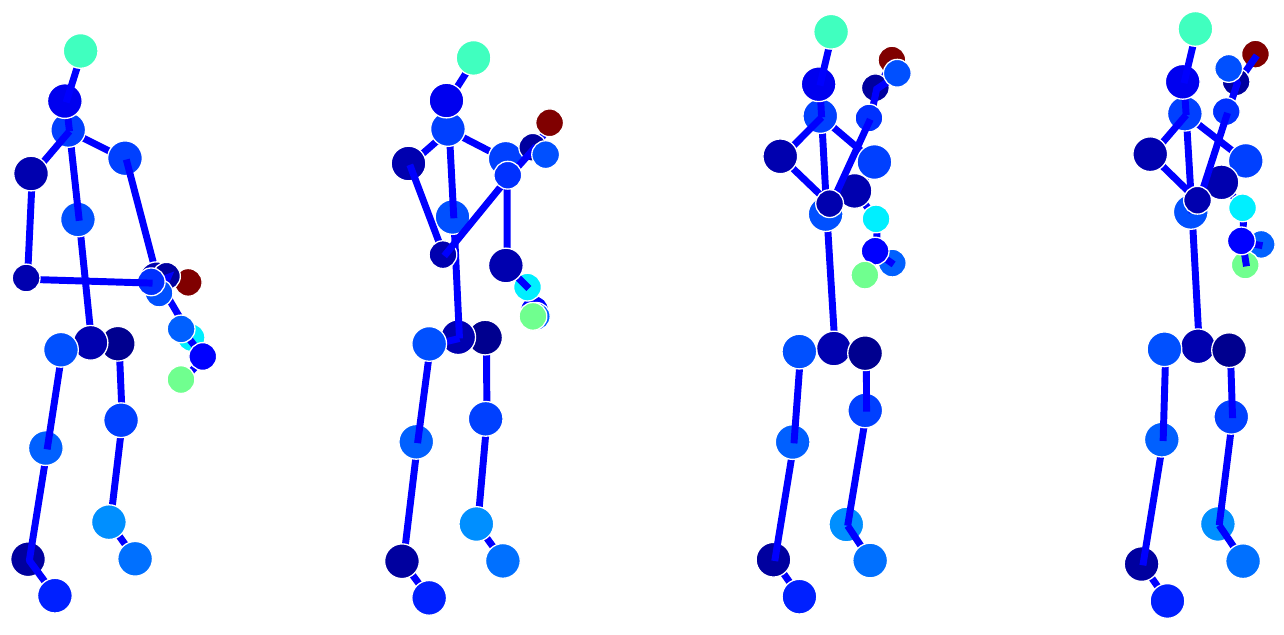}}
  \hspace{0.5in}
  \subfigure[Eat meal/snack]{
    \label{fig:N_01}
    \includegraphics[width=2.8in]{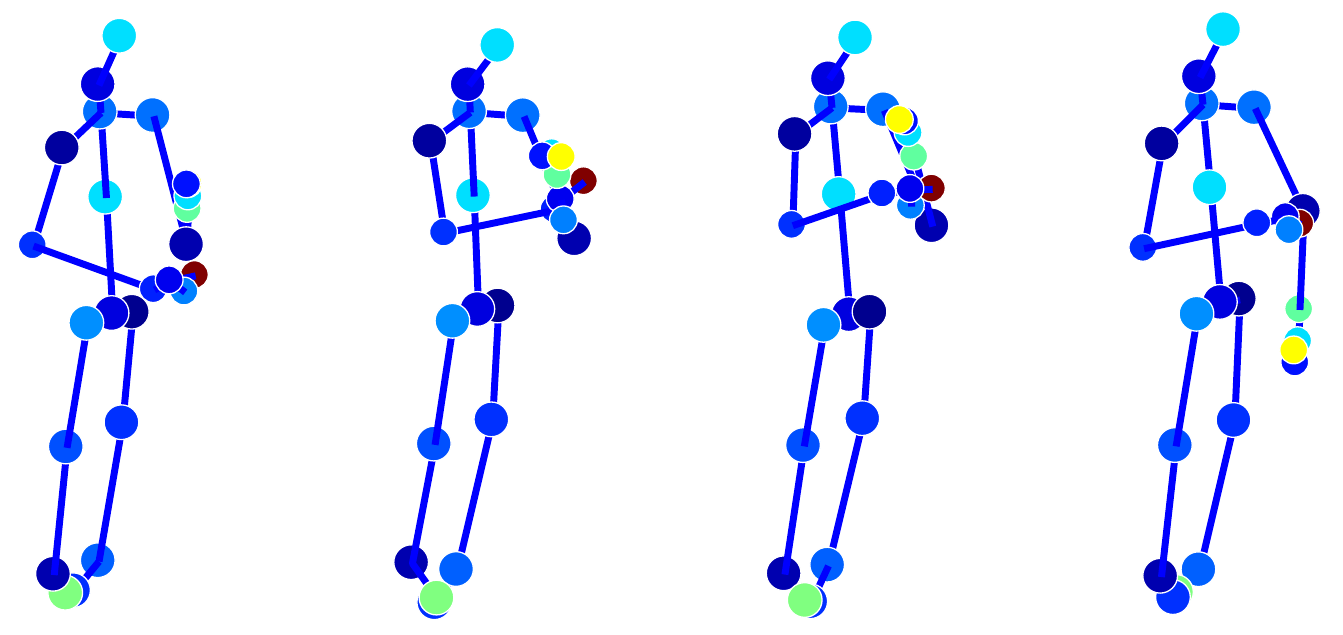}}
  \subfigure[Brushing teeth]{
    \label{fig:N_02}
    \includegraphics[width=2.8in]{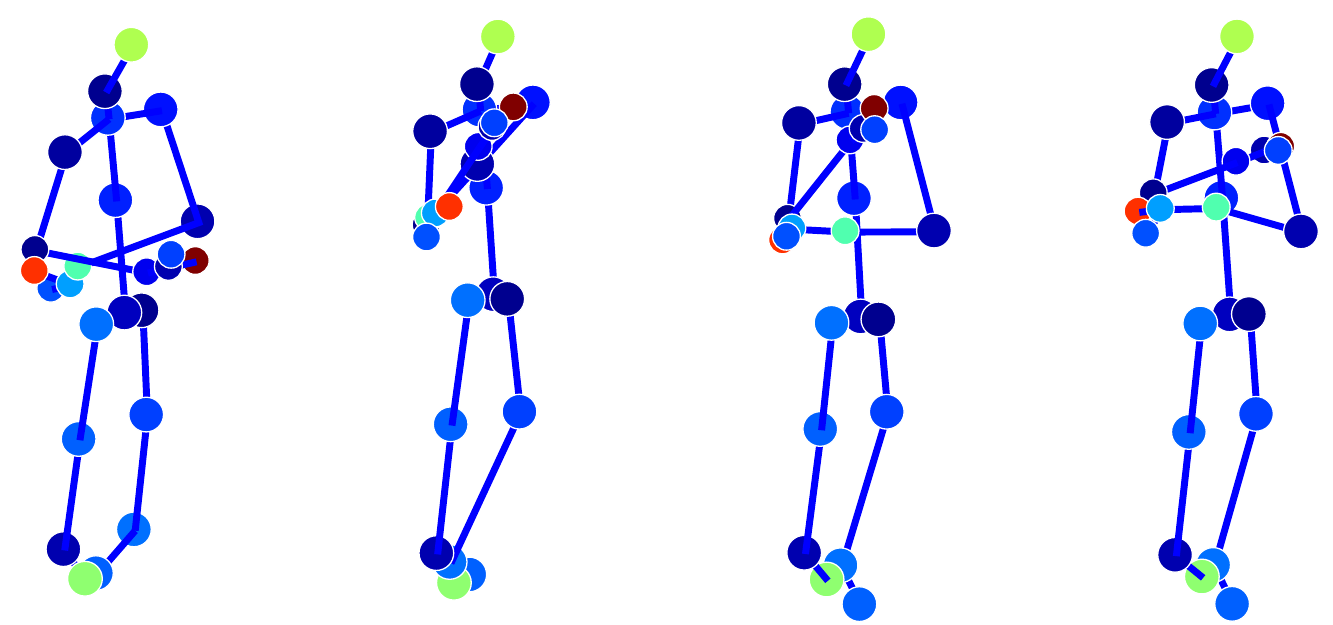}}
  \hspace{0.5in}
  \subfigure[Use a fan (with hand or paper)]{
    \label{fig:N_48}
    \includegraphics[width=2.8in]{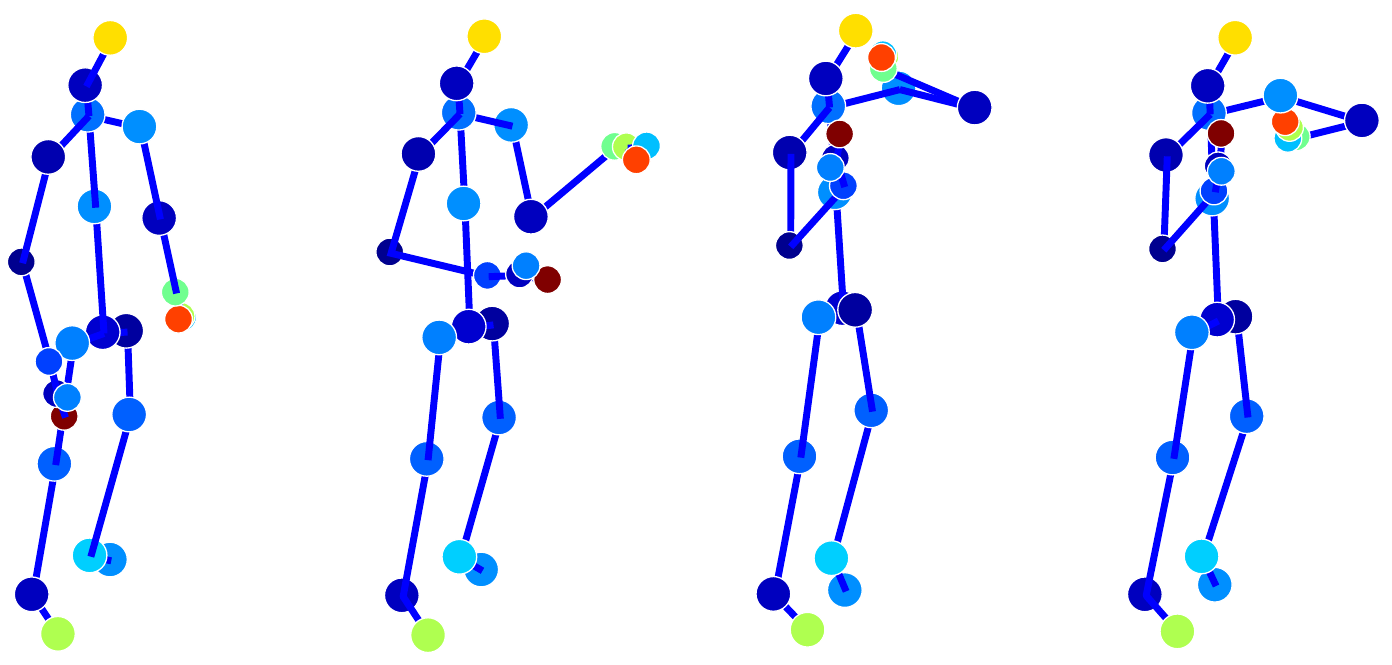}}
  \subfigure[Pickup]{
    \label{fig:N_05}
    \includegraphics[width=2.8in]{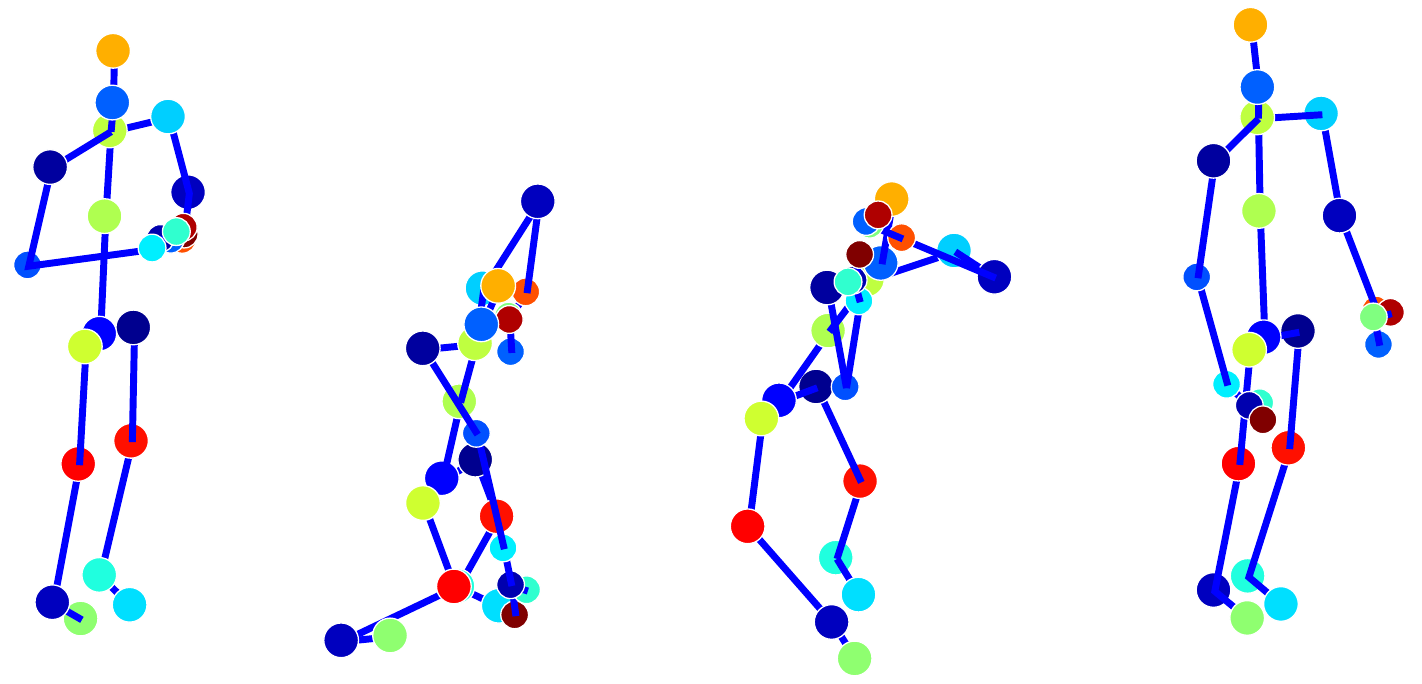}}
  \hspace{0.5in}
  \subfigure[Throw]{
    \label{fig:N_06}
    \includegraphics[width=2.8in]{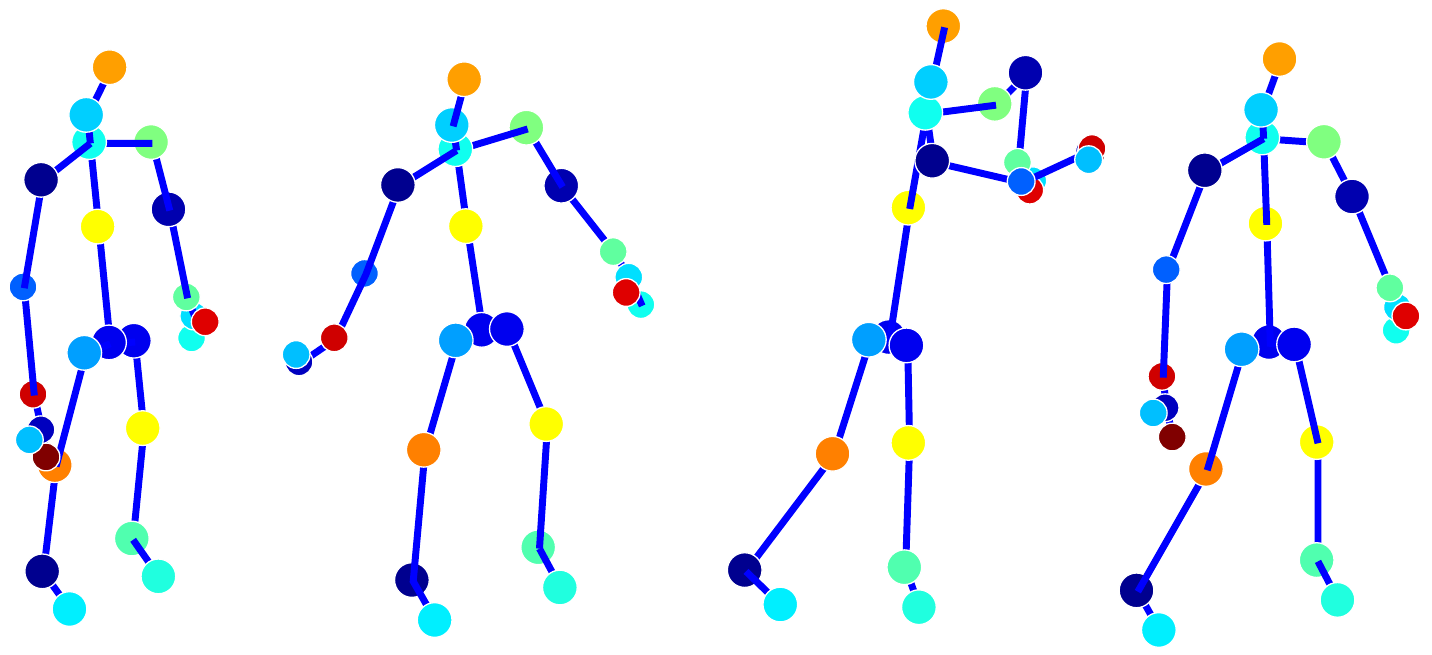}}
  \subfigure[Wear a shoe]{
    \label{fig:N_15}
    \includegraphics[width=2.8in]{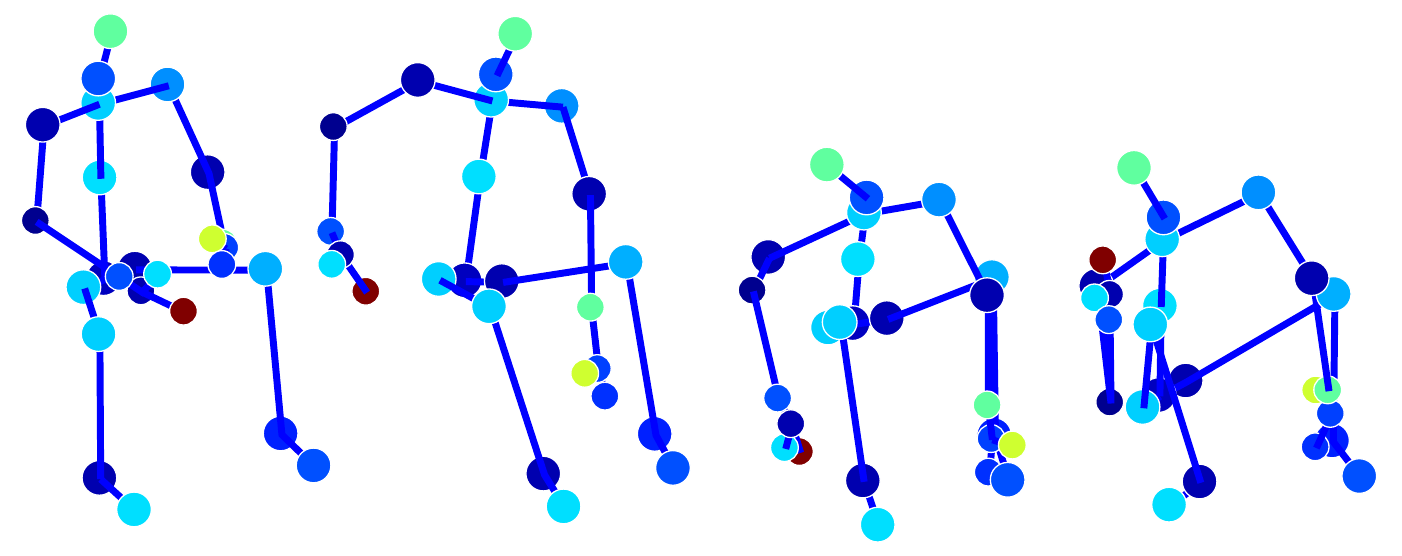}}
  \hspace{0.5in}
  \subfigure[Cheer up]{
    \label{fig:N_21}
    \includegraphics[width=2.8in]{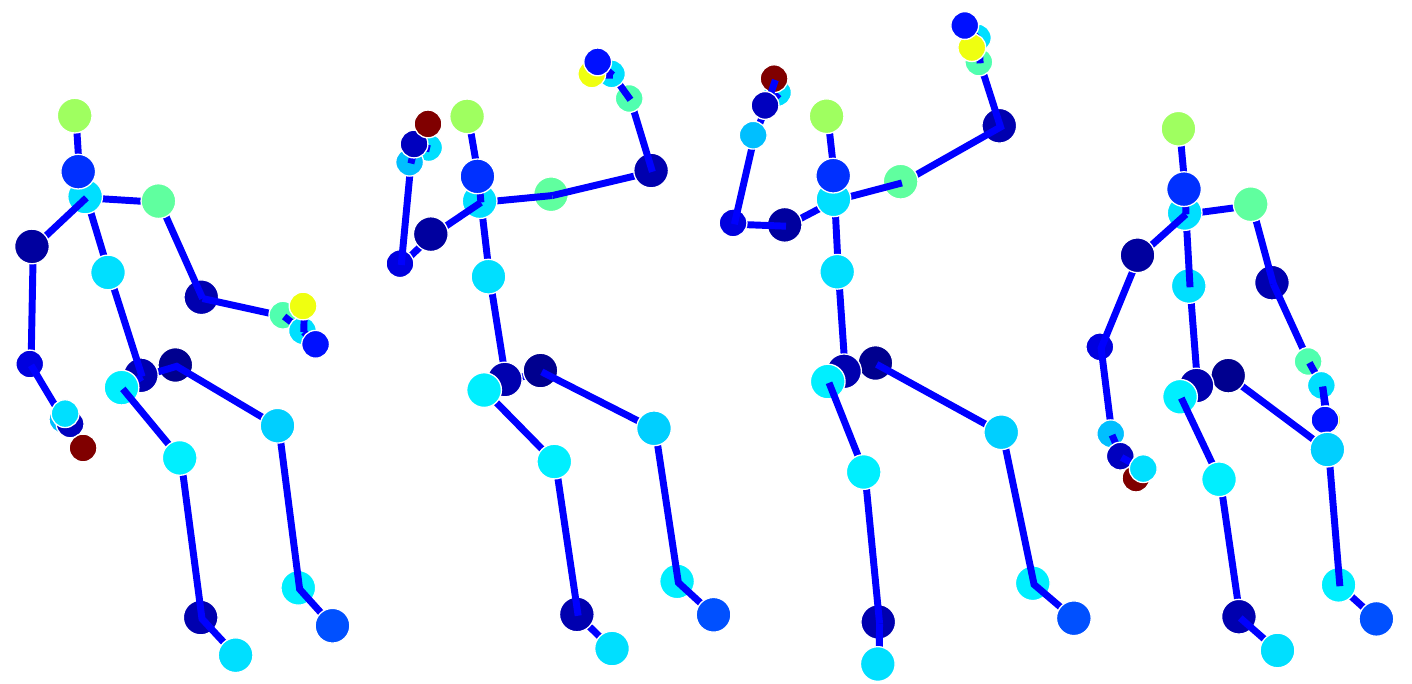}}
  \subfigure[Hopping (one foot jumping)]{
    \label{fig:N_25}
    \includegraphics[width=2.8in]{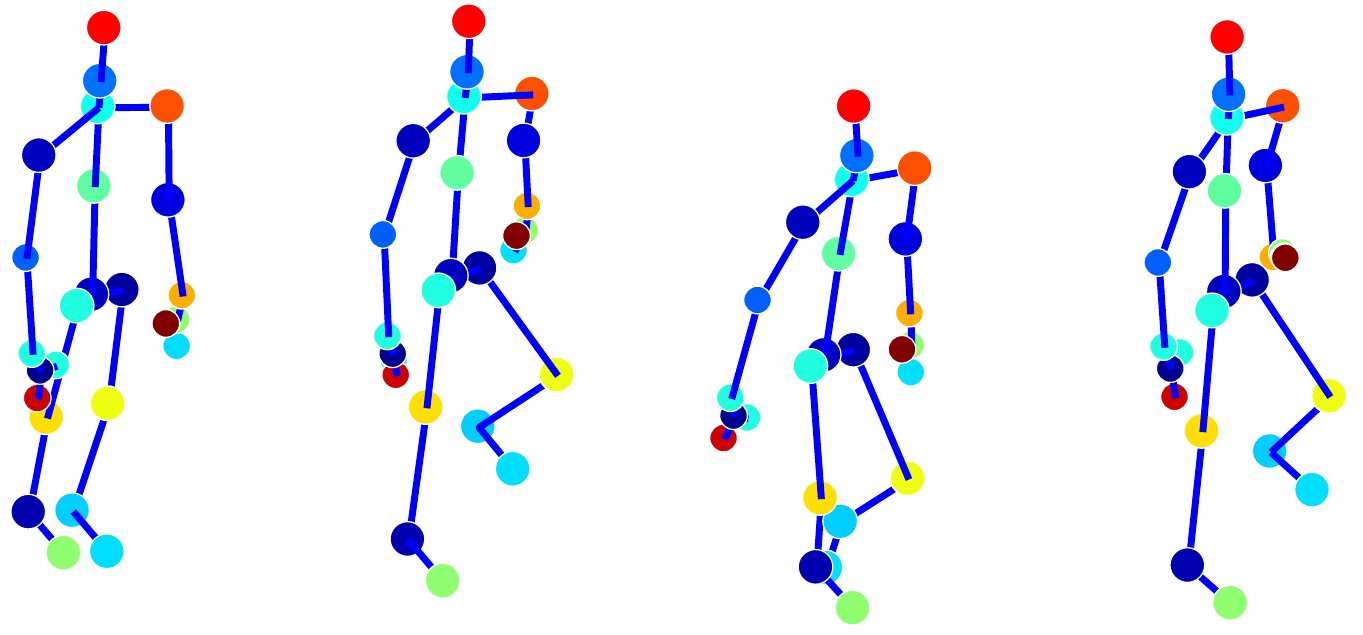}}
  \hspace{0.5in}
  \subfigure[Jump up]{
    \label{fig:N_26}
    \includegraphics[width=2.8in]{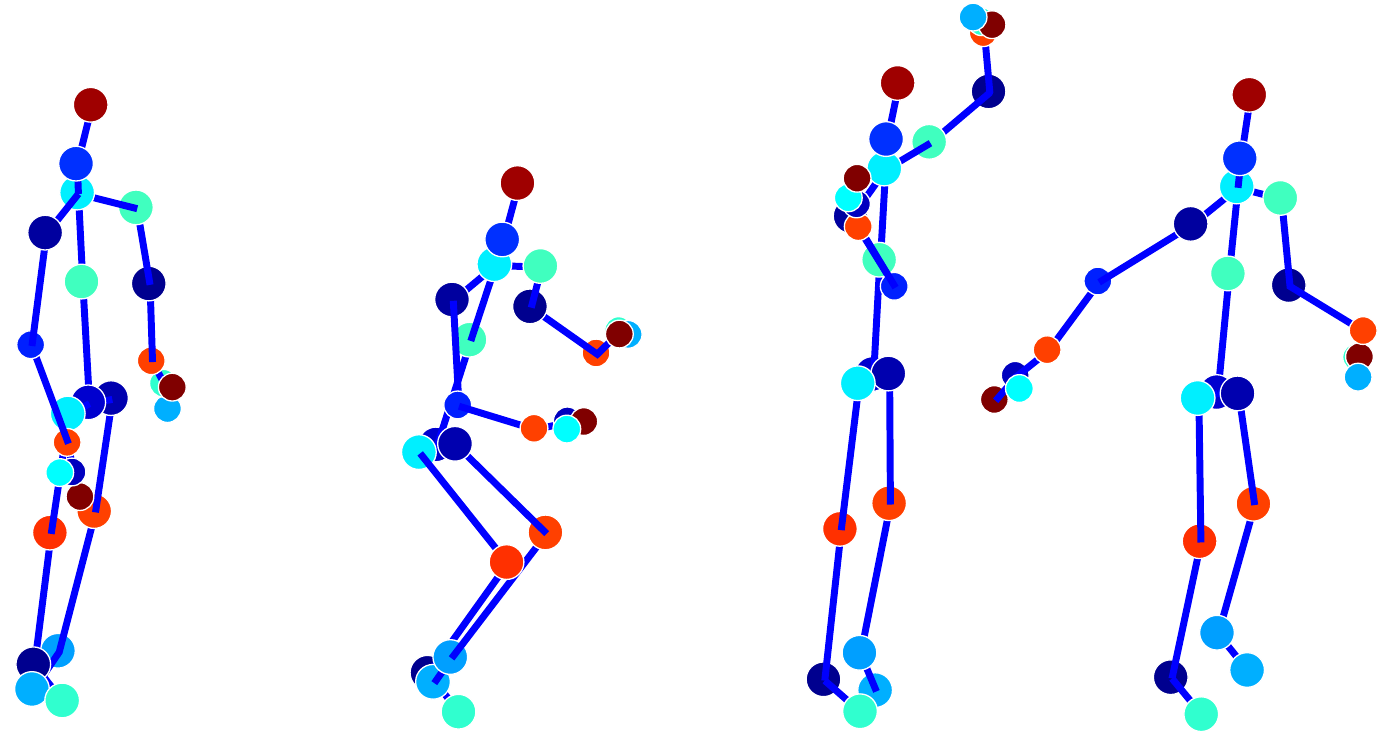}}
  \par
  \subfigure[Colorbar]{
    \includegraphics[width=2.5in]{colorbar.pdf}}
  \caption{Visual examples of detected salient action units on NTU RGB+D. Higher weights in the colorbar means more important to characterize the action.}
  \label{fig:visual:NTU}
\end{figure*}

\begin{table*}[t]
  \caption{Results of tuning network modules. }
  \label{tab:cmp-exp}
  \centering
  \renewcommand\arraystretch{1.2}
    \begin{tabular}{l|c|c|c|c|c|c|c|c|c}
    \hline\hline
    \multirow{3}{*}{Method}&\multicolumn{2}{c|}{HDM05}&\multirow{2}{*}{Florence}&\multicolumn{4}{c|}{LSC}&\multicolumn{2}{c}{NTU RGB+D}\\
    \cline{2-3}\cline{5-10}&Protocol~\cite{wang2015beyond}&Protocol~\cite{huang2017riemannian}&&\multicolumn{2}{c|}{Cross Sample} &\multicolumn{2}{c|}{Cross Subject}&Cross Subject&Cross View\\
    \cline{2-10}&Acc.&Acc.&Acc.&Prec.&Rec.&Prec.&Rec.&Acc.&Acc.\\
    \cline{1-10}
    A$^2$GNN $\circleddash$ filt. $\circleddash$ AU det.
    & 71.67\% & 77.74\%$\pm$1.61 & 93.49\% & 79.03\% & 76.96\% & 73.21\% & 70.77\% & 63.18\% & 68.82\% \\
    A$^2$GNN $\circleddash$ AU det.
    & 75.28\% & 83.01\%$\pm$1.43 & 96.74\% & 86.56\% & 85.71\% & 79.97\% & 79.92\% & 68.26\% & 80.08\% \\
    A$^2$GNN
    & 76.46\% & 84.47\%$\pm$1.52 & 98.60\% & 87.61\% & 88.10\% & 83.98\% & 82.03\% & 72.74\% & 82.80\% \\
    \hline\hline
    \end{tabular}
\end{table*}

\subsection{Analysis of Confusion Matrices}\label{sec:confusion}

To further reveal what classes are easy to confuse with others, we show confusion matrices on HDM05, NTU RGB+D and Florence datasets, repectively in Fig.~\ref{fig:HDM05}, Fig.~\ref{fig:NTU}, and Fig.~\ref{fig:Florence}. For LSC dataset, we don't depict its confusion matrix due to the different evaluation criterion (precision and recall).

For the HDM05 dataset, as shown in Fig.~\ref{fig:HDM05:All}, we give the confusion matrix of 130 classes in the case of recognition rate 84.47\% (\ie, the testing protocol follows the literature~\cite{huang2017riemannian}). The diagonal characterizes the correct classification for each action, and those non-diagonals depict the confusion results cross different classes. In order to view them more clearly, we provide two close up areas (Part A and Part B) in Fig.~\ref{fig:HDM05:A} and Fig.~\ref{fig:HDM05:B}. As observed from the two subfigures, our proposed method suffers some failures in distinguishing those very similar actions. For example, the depositing and grabbing are seriously confused, as ``depositFloorR" and ``grabFloorR" are almost visually consistent in knees bent and arms stretch. For a basic kicking behavior, the actions of left foot forwards (``kickLFront1Reps") or left foot sideways (``kickLSide1Reps") are subtle with small inter-class distance. Likewise, there are some other confusable actions, such as ``kickRFront1Reps'' vs ``kickRSide1Reps", ``punchLFront1Reps" vs ``punchLFront2Reps" and so on. Note that the postfixes ``1Reps" and ``2Reps" means the repetitive times of a action.

Different from HDM05 with finer-partitioned action classes, NTU RGB+D contains many pairs of reversed actions, such as ``wear A" vs ``take off A", ``put on A" vs ``take off A", etc. As shown in Fig.~\ref{fig:NTU:1}, these reversed pairs are easy to be confused when only considering human skeleton data. Besides, similar to the observation from HDM05, the confusion cases occur in those similar motions, such as ``pat on back of other person" vs ``point finger at the other person". Likewise, this phenomenon happens in Florence 3D, as shown in Fig.~\ref{fig:Florence}. Although the accuracies of 100\% are achieved on seven actions (\ie, wave, answer phone, clap, tight lace, sit down, stand up, bow), there are a few uncorrect classifications including ``drink from a bottle'' to ``answer phone", ``read watch" to ``wave"/``answer phone". A future possible solution to these above confusable phenomenons is that the contextual information in the RGB color space may be considered to compensate for the plain coordinate information of skeletal joints to some extent.

\subsection{Visualization of Action Units}\label{sec:visual}

To verify whether the action-attending layer detects those salient action units, we provide some visualization examples in Fig.~\ref{fig:visual:Florence} and Fig.~\ref{fig:visual:NTU}, where skeletal joints are colored according to the learnt weights. Note that here we exhibit
the first action-attending layer because the learnt weights are directly associated with skeletal joints. From the visualization of Florence 3D in Fig.~\ref{fig:visual:Florence}, we can find that our A$^2$GNN is able to learn those salient joints for all 9 actions. For examples, for the ``wave" action, those joints on the moving right arm are important; for the ``tight lace" action, the joints of hands and knees (bent) is endowed with higher weights. Specifically, for the same ``wave" action, our method can still correctly detect the moving units of the left arm or the right arm, as shown in Fig.~\ref{fig:F_9} and Fig.~\ref{fig:F_10}.

For the more complex action dataset, NTU RGB+D, we can still observe that those detected salient action units are almost matched to our intuitive understanding. For the ``pick up" action, the units of head, hands and knees are crucial to identify this action. For the ``hopping" action (right foot jumping) in Fig.~\ref{fig:N_25}, besides those salient joints on head, hands and knees, the joint on left shoulder is still more salient than that of right shoulder due to the more drastic motion variations on left shoulder compared to right shoulder. In contrast to Florence 3D, as NTU RGB+D is configured with more sensors to record human motions, thus more detailed motions can be captured, \eg, the foe joint on the ``wear a shoe" action, the ``hopping" action.

Consequently, the above observations indicate that the action-attending layer can adaptively weight skeletal joints for different actions, and the detected salient action units almost conform to our cognitive understanding on human actions. Moreover, the detection of action units can bring some gains of the accuracy as analyzed in Section~\ref{sec:params_analy}.

\subsection{Tuning of Our Network}\label{sec:params_analy}

\begin{figure}[t]
  \centering
  \includegraphics[width=3in]{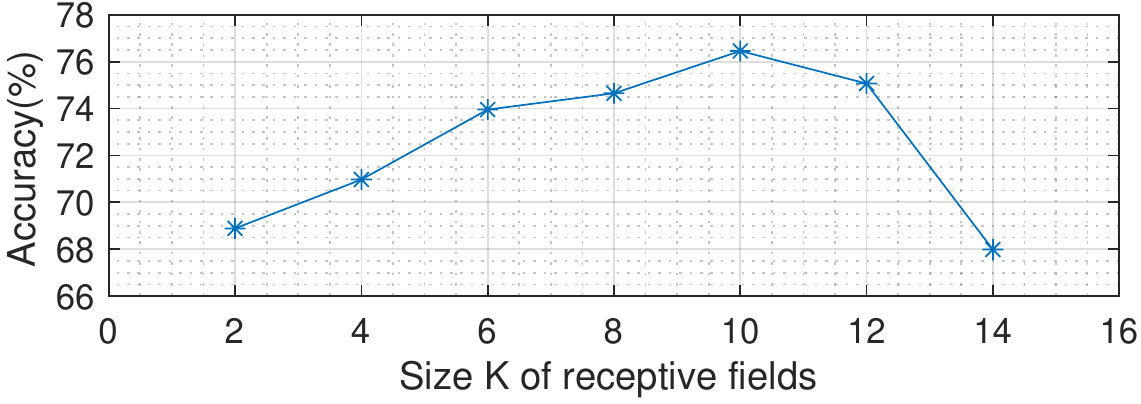}
  \caption{The performance trend of different receptive fields $K$ on HDM05.}
  \label{fig:params_K}
\end{figure}

In our proposed A$^2$GNN, there are three main modules: spectral graph filtering, action unit detection, and temporal motion modeling. The last module is an indispensable unit to our network. To dissect our network architecture, we conduct experiments on the four datasets by removing the former two modules (\ie, A$^2$GNN $\circleddash$ filt. $\circleddash$ AU det.) or only removing AU detection (\ie, A$^2$GNN $\circleddash$ AU det.).
The comparison results are reported in Table~\ref{tab:cmp-exp}. As observed from this table, the detection of action units can be helpful for action recognition more or less, as skeletal joints are successfully activated with different weights as analyzed in Section~\ref{sec:visual}. Further, the spectral graph filtering module plays a tremendous role in promoting recognition accuracy, as exhibited at the former two rows in the table. The main reason should be that high-level semantic features are extracted from graphs like the conventional convolution network on grid-shape images.

Among the parameters of our network, the most key is the size $K$ of receptive fields. With the increase of $K$, the local filtering region, \ie, covering the hopping neighbors, will become larger. To check the influence of different $K$ values, we conduct an experiment on HDM05 dataset by testing $K =2,4,6,8,10,12,14$. The results are shown in Fig.~\ref{fig:params_K}, we can find that, the performance improves with the increase of $K$ due to larger receptive fields, and then degrades after $K=10$, for which a possible reason is the locality of salient action units.

\section{Conclusion}\label{sec:conclusion}

In this paper, an end-to-end action-attending graphic neural network (A$^2$GNN) is proposed to deal with the task of skeleton-based action recognition. In order to extract deep features from body skeletons, we model human skeletons into undirected attribute graphs, and then perform spectral graph filtering on skeletal graphs like the standard CNN. To detect those salient action units crucial to identify human motions, we further design an action-attending layer to adaptively weight skeletal joints for different actions. Those extracted deep graph features at consecutive frames are finally fed into a recurrent network of LSTM. Extensive experiments and analyses have indicated that the modules of spectral graph filtering and action unit detection play an important role in the improvement on action classification. Especially, the action-attending layer also produces some interesting salient action units, which may be understandable from the view of our cognition. Further, our proposed A$^2$GNN has achieved the state-of-the-art results on the four public skeleton-based action datasets, including the current largest and most challenging NTU RGB+D dataset. In the future, we will explore the fusion RGB color information into our network.

% if have a single appendix:
%\appendix[Proof of the Zonklar Equations]
% or
%\appendix  % for no appendix heading
% do not use \section anymore after \appendix, only \section*
% is possibly needed

% use appendices with more than one appendix
% then use \section to start each appendix
% you must declare a \section before using any
% \subsection or using \label (\appendices by itself
% starts a section numbered zero.)
%

%\appendices
%\section{Proof of the First Zonklar Equation}
%Appendix one text goes here.
%
%% you can choose not to have a title for an appendix
%% if you want by leaving the argument blank
%\section{}
%Appendix two text goes here.

% use section* for acknowledgment
%\section*{Acknowledgment}

%The authors would like to thank...

% Can use something like this to put references on a page
% by themselves when using endfloat and the captionsoff option.
\ifCLASSOPTIONcaptionsoff
  \newpage
\fi

% trigger a \newpage just before the given reference
% number - used to balance the columns on the last page
% adjust value as needed - may need to be readjusted if
% the document is modified later
%\IEEEtriggeratref{8}
% The "triggered" command can be changed if desired:
%\IEEEtriggercmd{\enlargethispage{-5in}}

% references section

% can use a bibliography generated by BibTeX as a .bbl file
% BibTeX documentation can be easily obtained at:
% http://mirror.ctan.org/biblio/bibtex/contrib/doc/
% The IEEEtran BibTeX style support page is at:
% http://www.michaelshell.org/tex/ieeetran/bibtex/
%\bibliographystyle{IEEEtran}
% argument is your BibTeX string definitions and bibliography database(s)
%\bibliography{IEEEabrv,../bib/paper}
%
% <OR> manually copy in the resultant .bbl file
% set second argument of \begin to the number of references
% (used to reserve space for the reference number labels box)
%%\begin{thebibliography}{1}
%%
%%\bibitem{IEEEhowto:kopka}
%%H.~Kopka and P.~W. Daly, \emph{A Guide to \LaTeX}, 3rd~ed.\hskip 1em plus
%%  0.5em minus 0.4em\relax Harlow, England: Addison-Wesley, 1999.
%%
%%\end{thebibliography}

\bibliographystyle{IEEETran}
\bibliography{egbib}
\end{document}